%% file: sample_paper.tex
\documentclass[twoside]{article}

\usepackage[accepted]{aistats2025}

\usepackage[utf8]{inputenc} 
\usepackage[T1]{fontenc}    
\usepackage{hyperref}       
\usepackage{url}            
\usepackage{booktabs}       
\usepackage{amsfonts}       
\usepackage{nicefrac}       
\usepackage{microtype}      
\usepackage{xcolor}         
\usepackage{graphicx}
\usepackage{amsmath}
\usepackage{amsthm}
\usepackage{bbm}
\usepackage{nicematrix}
\usepackage{multirow}
\usepackage{caption}
\usepackage{subcaption}
\usepackage{tikz}
\usepackage{tcolorbox}
\usepackage{xcolor}
\usepackage{sidecap}
\sidecaptionvpos{figure}{c}
\usepackage[frozencache,cachedir=.]{minted}
\tcbuselibrary{minted,breakable,xparse,skins}

\newtheorem{theorem}{Theorem}[section]
\newtheorem{lemma}[theorem]{Lemma}
\newtheorem{definition}[theorem]{Definition}

\newtheorem{corollary}[theorem]{Corollary}

\DeclareMathOperator*{\argmin}{arg\,min}

\newcommand\ev{\mathbf{e}}

\newcommand\pv{\mathbf{p}}
\newcommand\qv{\mathbf{q}}

\newcommand\uv{\mathbf{u}}
\newcommand\vv{\mathbf{v}}

\newcommand\xv{\mathbf{x}}
\newcommand\yv{\mathbf{y}}

\newcommand\Hv{\mathbf{H}}
\newcommand\Iv{\mathbf{I}}

\newcommand\Pv{\mathbf{P}}

\newcommand\Rv{\mathbf{R}}

\newcommand\Tv{\mathbf{T}}
\newcommand\Uv{\mathbf{U}}
\newcommand\Vv{\mathbf{V}}

\newcommand\Xv{\mathbf{X}}
\newcommand\Yv{\mathbf{Y}}
\newcommand\Zv{\mathbf{Z}}

\begin{document}

%

%
\runningauthor{Gabriel Moreira, Manuel Marques, João Paulo Costeira and Alexander Hauptmann}

\twocolumn[

\aistatstitle{Learning Visual-Semantic Subspace Representations}

\aistatsauthor{Gabriel Moreira \And Manuel Marques}
\aistatsaddress{Instituto Superior Técnico \\ Carnegie Mellon University \And  Institute for Systems and Robotics \\ Instituto Superior Técnico} 

\aistatsauthor{João Paulo Costeira \And Alexander Hauptmann}
\aistatsaddress{Institute for Systems and Robotics \\ Instituto Superior Técnico \And  Carnegie Mellon University} ]

\begin{abstract}
Learning image representations that capture rich semantic relationships remains a significant challenge. Existing approaches are either contrastive, lacking robust theoretical guarantees, or struggle to effectively represent the partial orders inherent to structured visual-semantic data. In this paper, we introduce a nuclear norm-based loss function, grounded in the same information theoretic principles that have proved effective in self-supervised learning. We present a theoretical characterization of this loss, demonstrating that, in addition to promoting class orthogonality, it encodes the spectral geometry of the data within a subspace lattice. This geometric representation allows us to associate logical propositions with subspaces, ensuring that our learned representations adhere to a predefined symbolic structure.
\end{abstract}

\section{INTRODUCTION}
In this paper, we present a methodology for learning image representations that adhere to a predefined symbolic structure. Consider a representation space $\mathcal{H}$ and an image encoded therein. Answering any Boolean proposition conditioned on it amounts to specifying a binary value and corresponds thus, to a subset of an observation space $\mathcal{Y}$. Intuitively, we expect the representation space and the observation space to be correlated -- if proposition $p$ entails $q$, the respective observations of $p$ and $q$ in $\mathcal{Y}$ preserve the set-theoretic inclusion, and so should their corresponding representations in $\mathcal{H}$.  Notably, works on semantic linguistics have even conceptualized representations as regions, with the set-theoretic inclusion corresponding to semantic entailment \cite{geffet2005distributional,gardenfors2004conceptual}.

Despite their ubiquity, Euclidean embeddings with the usual inner product do not represent partial orders and symbolic operations in the most natural manner. While the commonly used cosine similarity encodes a notion of \textit{closeness}, useful for simple retrieval tasks, its symmetry imposes a latent space geometry that is \textit{incoherent} with that of the semantics \cite{patel2022modeling}. Indeed, state-of-the-art vision-language models trained via contrastive objectives exhibit poor performance at tasks requiring compositionality \cite{yuksekgonul2022and} and logical reasoning, such as understanding negations \cite{singh2024learn,quantmeyer2024and}. This disconnection limits our ability to harness structural constraints, compromising reasoning capabilities and hindering interpretability \cite{steck2024cosine}.

Previous efforts to incorporate symbolic structure into representation spaces, particularly focusing on transitive relations and the asymmetries present in ontologies and knowledge graphs, have proposed \textit{ad-hoc} contrastive losses or geometry-inspired approaches, such as measure-based embeddings \cite{ren2020beta,vilnis2018probabilistic,li2018smoothing}, quasimetrics \cite{memoli2018quasimetric} and hyperbolic representations \cite{moreira2024hyperbolic,Chami2020FromClustering,Chamberlain2017NeuralSpace,Nickel2017PoincareRepresentations,Atigh2021HyperbolicPrototypes,Ganea2018HyperbolicEmbeddings}. Although these approaches are effective, many of them lack robust theoretical guarantees.

In contrast, we demonstrate that a rich and consistent geometrical structure arises naturally from minimizing a nuclear norm-based loss, grounded in the same information theoretic principles that underlie recent advances in self-supervised learning \cite{bardesvicreg,lecomte2023information,yu2020learning,yerxa2024learning}. Our work contributes to the growing interest in utilizing rank-based non-contrastive losses for representation learning. While prior results focused on enforcing class orthogonality, we identify the minimizer of the proposed loss as a spectral embedding of the data, in the form of a Boolean subspace lattice \cite{mittelstaedt1972interpretation, birkhoff1975logic}. 

By encoding semantic partial orders, our representations enable a probabilistic formulation of propositional queries as subspaces \cite{vanRijsbergen_2004}. These, in turn, facilitate the representation of logical operations such as conjunction, disjunction and negation in terms of subspace calculus. Such a capability is crucial for complex symbolic reasoning and for faithfully representing relations that extend beyond mere similarity, such as hypernymy, inclusion, and entailment, which form the cornerstone of semantic understanding.

In summary, our contributions are three-fold:
\begin{enumerate}
    \item A new loss function that acts as a surrogate for the mutual information between the semantic, or symbolic, distribution and the embedding distribution (Section \ref{sec:motivation}).
    \item We prove that the minimization of the proposed loss guarantees that the visual representations adhere to the spectral geometry of the underlying semantics (Section \ref{sec:theoretical_results}).
    \item We show that the learned representations form a subspace Boolean lattice, where propositions are encoded as projection operators (Section \ref{sec:boolean_lattice}).
\end{enumerate}

Our general learning framework yields representations suitable for single-label and multi-label classification, as well as for retrieval from complex propositional queries. While we present our approach in the context of visual-semantic data, any modalities may be considered. Code is available at \url{https://github.com/gabmoreira/subspaces}.

\section{RELATED WORK}
\paragraph{Structured Representations} A large number of works have addressed the problem of endowing representation spaces with an interpretable structure, capable of encoding the partial orders inherent to knowledge graphs and image-caption hierarchies \cite{vendrov2015order}. A notable subset of these include measure-based representations, such as Gaussian embeddings \cite{vilnis2014word}, Gaussian mixtures \cite{choudhary2021probabilistic}, box embeddings \cite{vilnis2018probabilistic,li2018smoothing,patel2022modeling} and Beta embeddings \cite{ren2020beta}. An alternative direction has been to look at the problem from a geometrical standpoint, namely via negatively curved manifolds. In fact, hyperbolic embeddings have seen widespread adoption in the literature, owing to their ability to naturally represent tree-like structures \cite{moreira2024hyperbolic,Nickel2018LearningGeometry,Chamberlain2017NeuralSpace,Nickel2017PoincareRepresentations,Chami2020FromClustering,Atigh2021HyperbolicPrototypes,Ganea2018HyperbolicEmbeddings}. 

\paragraph{Information Theoretic Representations} Our work bridges the gap between the aforementioned research on structured representations and works on representation learning based on information theoretic principles, such as MCR$^2$ \cite{yu2020learning} and MMCR \cite{yerxa2024learning,lecomte2023information}. The latter, in particular, while finding applications in self-supervised learning, uses a nuclear norm-based loss similar to the regularization proposed in OLÉ \cite{lezama2018ole}, as a means to improve image classification. Both OLÉ and MMCR promote intra-class low-rank and inter-class high-rank via the nuclear norm and either enforce unit-norm embeddings or impose a lower bound on the intra-class nuclear norm to avoid representation collapse. In contrast, our method focuses on the low-rank assumption of the joint visual-semantic distribution and we provide a theoretical characterization of the solution for the general multi-label setting, instead of considering only the disjoint case.


\section{LEARNING SUBSPACES VIA THE NUCLEAR NORM}
\subsection{A Joint Low-Rank Formulation}
\label{sec:motivation}
In this section, we informally derive the proposed loss, by establishing a connection between multi-label classification via low-rank matrix completion \cite{cabral2011matrix, goldberg2010transduction} and recent works that employ the nuclear norm as a loss \cite{yerxa2024learning,lecomte2023information}, or a regularization thereof \cite{lezama2018ole}. Consider an image dataset with $c$ labels and $n$ images. Both the images and the labels shall be represented in some $d$-dimensional vector space. We define the label matrix as $\Yv = [\yv_1\dots\yv_n]\in\{0,1\}^{c\times n}$, and assume that $\Yv$ contains all the structure that we want to impose on the representations $\Xv=[\xv_1\dots\xv_n]\in\mathbb{R}^{d\times n}$. If the latter are amenable to a linear probe, the joint distribution
\begin{equation}
\Zv:=\begin{bmatrix}\Yv\\\Xv\end{bmatrix}\in\mathbb{R}^{(c+d)\times n}
\end{equation}
is low-rank. Inferring missing labels, or class prototypes, can thus be formulated as finding a complete low-rank matrix that is close to $\Zv$ (for some metric) in all the observed entries \cite{cabral2011matrix,goldberg2010transduction}.

If we assume instead that we have access to $\Yv$ and that we wish to learn $\Xv$, the problem becomes one of learning representations. In order to avoid it being ill-posed \textit{i.e.}, representation collapse, we need to ensure that $\Xv$ is non-trivial. This can be achieved by forcing the representations to be full-rank. Using the nuclear norm $\|\cdot\|_\ast$ as a rank-surrogate, the proposed loss function reads as
\begin{equation}
    l(\Xv):=\left\|\Zv\right\|_\ast - \alpha \|\Xv\|_\ast + \beta \|\Xv\|_2^2,
    \label{eq:the_loss}
\end{equation}
for some $\beta \in (0,1)$ and $\alpha \in (0,1)$, and where $\|\cdot\|_2$ is the spectral norm. The low-rank we seek for $\Zv$ encodes the fact that $\Xv$ should require no more information than that specified in $\Yv$ \textit{i.e.}, $\Zv$ should be be redundant. Conversely, the high rank we require for $\Xv$ forces it to be informative, unless $\Yv$ is given. This objective is \textit{tantamount} to requiring the joint distribution of $\Xv,\Yv$ to have low entropy and $\Xv$ to have high entropy, which corresponds to minimizing $H(\Xv,\Yv) - H(\Xv) - H(\Yv)$ \textit{i.e.}, the negative mutual information between $\Yv$ and the representations $\Xv$. As we state formally in Section \ref{sec:theoretical_results}, minimizing this proxy of the negative mutual information yields a spectral embedding of $\Yv$.

\paragraph{Minimizing the Loss} One of our key results rests on the minimization of the loss $l(\mathbf{X})$, formally stated by Theorem \ref{theorem:min_nuclear_norm_ell2_squared}, in Section \ref{sec:theoretical_results}. We show that the representations that minimize (\ref{eq:the_loss}) are proportional to the eigenvectors of the Gram matrix of the labels $\Yv^\top\Yv$, up to an orthogonal transformation.

The $\ell_2$-penalty ensures that the top singular values of the minimizer $\Xv^\ast$ are equal. The hyperparameter $\alpha < 1$ guarantees that the rank of the minimizers is exactly that of $\Yv$. Table \ref{tab:losses_comparison} provides a summary of different nuclear norm-based losses and their minimizers, highlighting the contribution of each term. 

\begin{table*}
\caption{Losses and minimizers for $\Yv = \Uv_Y \mathbf{\Sigma}_Y \Vv_Y^\top\in\mathbb{R}^{c\times n}$ rank-$c$ and $\Xv\in\mathbb{R}^{d\times n}$, with $d \geq c$.}
\renewcommand{\arraystretch}{1.2}
    \centering
    \begin{tabular}{lcc}
         \toprule
         Loss $l(\Xv)$ & Dim & Argmin \\
         \hline
         \multirow{2}{*}{$\|\Zv\|_\ast - \alpha \|\Xv\|_\ast$ } & $d=c$ & $\frac{\alpha}{\sqrt{1-\alpha^2}}\Uv \mathbf{\Sigma}_Y \Vv_Y^\top \;:\; \Uv \in O(d)$ \\
         & $d>c$ & $\Uv \begin{bmatrix}\frac{\alpha}{\sqrt{1-\alpha^2}}\mathbf{\Sigma}_Y & \mathbf{0} \\ \mathbf{0} & \mathbf{0}\end{bmatrix} \begin{bmatrix}\Vv_Y^\top \\ \Vv^\top\end{bmatrix} :\; \Uv \in O(d), \Vv \in \mathcal{N}(\Vv_Y)$\\
         \hline
         \multirow{2}{*}{$\|\Zv\|_\ast - \|\Xv\|_\ast + \beta \|\Xv\|_2$} & $d=c$ & $\Uv (t^\ast \Iv_d)\Vv_Y^\top : \; \Uv\in O(d)$\\
         & $d>c$ & $ \Uv \begin{bmatrix}t^\ast \Iv_c & \mathbf{0} \\ \mathbf{0} & \mathbf{\Sigma}\end{bmatrix} \begin{bmatrix} \Vv_Y^\top \\ \Vv^\top\end{bmatrix} :\; \Uv\in O(d), t^\ast \succeq \mathbf{\Sigma}, \Vv \in \mathcal{N}(\Vv_Y)$\\
         \hline
         \multirow{2}{*}{$\|\Zv\|_\ast - \alpha\|\Xv\|_\ast + \beta \|\Xv\|_2^2$} & $d=c$ & $\Uv (t^\ast \Iv_c)\Vv_Y^\top\; :\; \Uv\in O(d)$\\
         & $d>c$ & $ \Uv (t^\ast \Iv_c) \Vv_Y^\top :\; \Uv\in \mathrm{St}_{c}(\mathbb{R}^d)$\\
         \bottomrule
    \end{tabular}
    \label{tab:losses_comparison}
\end{table*}

\subsection{Theoretical Guarantees and Derivations}
\label{sec:theoretical_results}
Our objective for this section is formally derive the minimizer of ($\ref{eq:the_loss}$), with the required auxiliary results. We will denote by $O(d) = \{\Uv \in \mathbb{R}^{d\times d} : \Uv^\top \Uv = \Uv\Uv^\top = \Iv_d\}$ the orthogonal group. The Stiefel manifold of orthonormal $d$-frames in $\mathbb{R}^n$ reads as $\mathrm{St}_d(\mathbb{R}^n)\subset\mathbb{R}^{n\times d}$ and the nullspace of a matrix as $\mathcal{N}$. The nuclear and spectral norms are denoted as $\|\Xv\|_\ast$ and $\|\Xv\|_2$, respectively. Recalling that the nuclear norm is the sum of the singular values $\|\Xv\|_\ast = \mathrm{Tr}((\Xv^\top \Xv)^{1/2})$, we have that it is invariant to orthogonal transformations (proofs are deferred to Appendix \ref{sec:app_proofs}).

\begin{lemma}[Symmetry]
    \label{lemma:nuclear_symmetry}
    Let $\Yv\in\mathbb{R}^{c\times n}$ and $\Xv\in\mathbb{R}^{d\times n}$. For any $\Uv_1 \in O(c), \Uv_2\in O(d)$ and $\Vv\in O(n)$
    \begin{equation}
        \left\|\begin{bmatrix}\Uv_1 \Yv \\ \Uv_2 \Xv \Vv \end{bmatrix}\right\|_\ast = \left\|\begin{bmatrix}\Yv \Vv^\top \\ \Xv \end{bmatrix}\right\|_\ast
    \end{equation}
\end{lemma}

If we fix the singular values of $\Xv$, the invariance from Lemma \ref{lemma:nuclear_symmetry} turns the minimization of $\|\cdot\|_\ast$ over $\mathbb{R}^{d\times n}$ into a minimization over $\mathrm{St}_d(\mathbb{R}^n)$. The solution is given by the following Theorem.

\begin{theorem}
    \label{theorem:nuclear_min_v}
    Let $\Yv\in\mathbb{R}^{c\times n}$ be a rank-$c$ matrix with SVD $\Uv_Y \mathbf{\Sigma}_Y \Vv_Y^\top$, where $\Uv_Y \in O(c)$, $\mathbf{\Sigma}_Y \in\mathbb{R}^{c\times c}$ is diagonal, with singular values $\mu_1 \geq \dots \geq \mu_c$, and $\Vv_Y \in \mathrm{St}_c(\mathbb{R}^n)$. For a rank-$d$ matrix $\Xv \in \mathbb{R}^{d\times n}$ with SVD $\Uv_X\mathbf{\Sigma}_X \Vv_X^\top$ and singular values $\sigma_1 \geq \dots \geq \sigma_d$, let 
    \begin{equation}
        \Vv^\ast = \argmin_{\Vv_X \in \mathrm{St}_d(\mathbb{R}^n)}\left\|\begin{bmatrix}\Yv \\ \Uv_X\mathbf{\Sigma}_X \Vv_X^\top \end{bmatrix}\right\|_\ast.
    \end{equation}
    It follows that, 
    \begin{itemize}
        \item if $c = d$,  then $\Vv^\ast = \Vv_Y$ and the min is $\sum_{i=1}^c \sqrt{\mu_i^2 + \sigma_i^2}$.
        \item  if $d < c$, then $\Vv^\ast =\Vv_Y  \begin{bmatrix} \Iv_d & \mathbf{0}_{d\times(c-d)}\end{bmatrix}^\top$ and the min is $\sum_{i=1}^{d} \sqrt{\mu_i^2 + \sigma_i^2} + \sum_{i=d+1}^c \mu_i$.
        \item  if $d > c$, then $\Vv^\ast = \begin{bmatrix} \Vv_Y & \Vv \end{bmatrix}$, with $\Vv^\top \Vv_Y = \mathbf{0}$, and the min is $\sum_{i=1}^c \sqrt{\mu_i^2 + \sigma_i^2} + \sum_{i=c+1}^{d}\sigma_i$.
    \end{itemize}
\end{theorem}

From Lemma \ref{lemma:nuclear_symmetry} and Theorem \ref{theorem:nuclear_min_v} we can derive our main result, which states that, for the appropriate choice of $\alpha$ and $\beta$, the minimizer of (\ref{eq:the_loss}) is a spectral embedding of $\Yv^\top \Yv$ with the rank of $\Yv$.

\begin{theorem}
    \label{theorem:min_nuclear_norm_ell2_squared}
    Let $\Yv \in \mathbb{R}^{c\times n}$ be a rank-$c$ matrix with SVD $\Uv_Y \mathbf{\Sigma}_Y \Vv_Y^\top$, where $\Uv_Y \in O(c)$, $\mathbf{\Sigma}_Y \in\mathbb{R}^{c\times c}$ is diagonal, with singular values $\mu_1 \geq \dots \geq \mu_c$,  and $\Vv_Y\in \mathrm{St}_c(\mathbb{R}^n)$. For $\beta\in (0,1)$ and $\sqrt{\max\left\{0,1-\frac{4\beta^2\mu_c^2}{c^2}\right\}} \leq \alpha < 1$, define the set
    \begin{equation}
        \mathcal{X} :=  \argmin_{\Xv \in \mathbb{R}^{d\times n}} \left\{\left\|\begin{bmatrix} \Yv \\ \Xv \end{bmatrix}\right\|_\ast - \alpha\|\Xv\|_\ast + \beta\|\Xv\|_2^2\right\},
    \end{equation}
    for $d\geq c$. Then,
    \begin{equation}
        \mathcal{X} =\left\{ \Uv (t^\ast \Iv_c) \Vv_Y^\top : \; \Uv\in \mathrm{St}_c(\mathbb{R}^d)\right\},
    \end{equation}
    where $t^\ast > 0$ is the solution to $\sum_{i=1}^c\frac{t}{\sqrt{\mu_i^2 + t^2}} = \alpha c - 2\beta t$.
\end{theorem}

As formalized in the following lemma, the $\ell_2$-penalty guarantees that the non-trivial singular values of the minimizers of (\ref{eq:the_loss}) are equal. Without this penalty, the minimizer set for $d=c$ is the orbit $\{\Rv\Yv\alpha/\sqrt{1-\alpha^2} \;:\; \Rv\in O(d)\}$, akin to MMCR.

\begin{lemma}[No $\ell_2$-penalty]
    \label{lemma:min_nuclear_norm_no_ell2}
    Let $\Yv \in \mathbb{R}^{c\times n}$ be a rank-$c$ matrix with SVD given by $\Uv_Y \mathbf{\Sigma}_Y \Vv_Y^\top$, where $\Uv_Y \in O(c)$, $\mathbf{\Sigma}_Y \in\mathbb{R}^{c\times c}$ is diagonal, with singular values $\mu_1 \geq \dots \geq \mu_c$, and $\Vv_Y\in \mathrm{St}_c(\mathbb{R}^n)$. For $\alpha\in(0,1)$, define the set
    \begin{equation}
        \mathcal{X} := \argmin_{\Xv \in \mathbb{R}^{d\times n}} \left\{\left\|\begin{bmatrix} \Yv \\ \Xv \end{bmatrix}\right\|_\ast - \alpha\|\Xv\|_\ast \right\},
    \end{equation}
    for $d \geq c$. Then,
    \begin{align}
        \mathcal{X} = \bigg\{\Uv \begin{bmatrix}\frac{\alpha}{\sqrt{1-\alpha^2}}\mathbf{\Sigma}_Y & \mathbf{0} \\ \mathbf{0} & \mathbf{0} \end{bmatrix} \begin{bmatrix}\Vv_Y^\top \\ \Vv^\top\end{bmatrix} : \nonumber \\ \Uv \in O(d), \Vv \in \mathcal{N}(\Vv_Y)\bigg\}.
    \end{align}
\end{lemma}

\begin{lemma}
    \label{lemma:min_nuclear_norm_ell2}
    Let $\Yv \in \mathbb{R}^{c\times n}$ be a rank-$c$ matrix with SVD $\Uv_Y \mathbf{\Sigma}_Y \Vv_Y^\top$, where $\Uv_Y \in O(c)$, $\mathbf{\Sigma}_Y \in\mathbb{R}^{c\times c}$ is diagonal, with singular values $\mu_1 \geq \dots \geq \mu_c$, and $\Vv_Y\in \mathrm{St}_c(\mathbb{R}^n)$. For $\beta\in(0,1)$, define the set
    \begin{equation}
        \mathcal{X} :=  \argmin_{\Xv \in \mathbb{R}^{d\times n}} \left\{\left\|\begin{bmatrix} \Yv \\ \Xv \end{bmatrix}\right\|_\ast - \|\Xv\|_\ast + \beta\|\Xv\|_2\right\},
    \end{equation}
    for $d\geq c$. Then,
    \begin{align}
        \mathcal{X} =\bigg\{ \Uv \begin{bmatrix}t^\ast \Iv_c & \mathbf{0} \\ \mathbf{0} & \mathbf{\Sigma}\end{bmatrix} \begin{bmatrix} \Vv_Y^\top \\ \Vv^\top\end{bmatrix} : \nonumber \\ \Uv\in O(d),\; t^\ast \Iv \succeq \mathbf{\Sigma},\;\Vv \in \mathcal{N}(\Vv_Y)\bigg\},
    \end{align}
    where $t^\ast > 0$ is the solution to $\sum_{i=1}^c\frac{t}{\sqrt{\mu_i^2 + t^2}} = c - \beta$.
\end{lemma}

\begin{corollary}[Orthogonal disjoint classes]
    Let $\Yv$ contain $n$ samples from $c$ disjoint classes and 
    \begin{equation}
        \Xv \in \argmin_{\Xv \in \mathbb{R}^{d\times n}} \left\{\left\|\begin{bmatrix} \Yv \\ \Xv \end{bmatrix}\right\|_\ast - \alpha \|\Xv\|_\ast+ \beta\|\Xv\|_2^2\right\}.
    \end{equation}
    Then, $\yv_i \neq \yv_j \implies \langle \xv_i, \xv_j \rangle = 0$.
    \label{corollary}
\end{corollary}

\paragraph{Comparison with OLÉ, MMCR and MCR$^2$} MMCR maximizes the nuclear norm of the matrix of $\ell_2$-normalized class centroids. The optimal representations for this loss have the same right-singular vectors as the label matrix $\Yv$, with the singular values of the former proportional to those of the latter. MMCR behaves thus similarly to the loss in the first row of Table \ref{tab:losses_comparison}. While it represents disjoint classes as orthogonal subspaces, the solution for correlated labels is unclear. The orthogonalization proposed in OLÉ acts as a regularizer for cross-entropy training, whereas the loss we propose (\ref{eq:the_loss}) is standalone. In MCR$^2$ \cite{yu2020learning}, disjoint classes are encoded as orthogonal subspaces, with their sum spanning the entire representation space. In our case, the corresponding subspaces have the smallest dimension necessary to represent $\Yv$ \textit{i.e.}, the rank of our representations is that of the semantics, regardless of the ambient space. This key difference can be attributed to the use of the nuclear norm which, being the convex envelope of the rank, naturally leads to low-rank solutions.

\paragraph{Comparison with Covariance Regularization} From Lemma \ref{lemma:nuclear_symmetry}, the orbit $\{\Rv\Xv : \Rv\in O(d)\}$ contains equivalent embedding matrices. Denote the SVD of $\Xv$ by $\Xv=\Uv\mathbf{\Sigma}\Vv^\top$. Since $\Uv\in O(d)$, we can define an equivalent embedding matrix by writing $\Xv$ in a new basis $\Uv$ such that its covariance is diagonal (assuming centered embeddings) \textit{i.e.}, $\Xv':=\Uv^\top \Xv$ and $\Xv'\Xv'^\top = \Uv^\top \Uv \mathbf{\Sigma} \Vv^\top \Vv \mathbf{\Sigma} \Uv^\top \Uv = \mathbf{\Sigma}^2$. Thus, instead of searching for $\Xv$ that minimizes $\|\Zv\|_\ast$ and maximizes $\|\Xv\|_\ast$, we can limit the search space to the set of matrices which admit the canonical basis of $\mathbb{R}^c$ as left singular vectors \textit{i.e.}, with an SVD $\Xv'=\mathbf{\Sigma} \Vv^\top$. For such matrices, the nuclear norm is given by the trace of $(\Xv\Xv^\top)^{\frac{1}{2}}$ and the problem of maximizing $\|\Xv\|_\ast$ becomes
\begin{align}
    &\max_{\Xv\in\mathbb{R}^{c\times n}}\; \sum_{i=1}^c \sqrt{(\Xv\Xv^\top)_{ii}}\nonumber \\
    &\;\;\;\mathrm{s.t.} \;\; (\Xv\Xv^\top)_{ij} = 0, \;i\neq j.
\end{align}
The terms of the sum are the standard deviations of each coordinate \textit{i.e.}, $\sqrt{(\Xv\Xv^\top)_{ii}}=\sqrt{\mathrm{Cov}[\Xv_{.,i}]}$. This is simply the covariance regularization proposed in VICReg \cite{bardesvicreg}.

\section{SUBSPACES AS PROPOSITIONS}
\label{sec:probabilistic_formulation}
Given a minimizer of (\ref{eq:the_loss}) we can derive probabilistic answers to propositions conditioned on the representations. This geometry-induced probabilistic formulation encodes the semantic partial orders, allowing us to perform propositional calculus on the representations, as illustrated in Fig. \ref{fig:projections}.

\begin{figure}
    \centering
    \input{projections.tikz}
    \caption{Subspace Boolean lattice. Each axis encodes a minterm of 2 literals: $\pv\land\qv$, $\neg\pv\land\qv$ and $\pv\land\neg\qv$. The propositions $\pv$ and $\qv$ are represented by two orthogonal 2-d subspaces. The squared norm of the projection of $\xv$, with $\|\xv\|=1$, over each subspace yields the posterior probability of the corresponding proposition.}
    \label{fig:projections}
\end{figure}

\subsection{Boolean Subspace Lattice}
\label{sec:boolean_lattice}
Given a full-row rank $\Yv\in\{0,1\}^{c\times n}$, the minimizer of (\ref{eq:the_loss}) embeds an inclusion Boolean sublattice of the power set $2^{[c]}$, as a subspace lattice of $\mathbb{R}^{d}$. We start by showing that, if we consider the columns of $\Yv=[\yv_1\dots\yv_n]$, denoted $\yv_i \in \{0,1\}^c$, as minterms of $c$ literals (logical propositions), then for any $\yv_i \neq \yv_j$, the corresponding embeddings are orthogonal \textit{i.e.}, $\langle \xv_i, \xv_j \rangle=0$. This allows us to associate more general propositions $\qv$ to subspaces $\mathcal{S}_q \subseteq \mathbb{R}^c$. Logical implication $\pv \implies \qv$ or equivalently, the semantic partial order $\qv \geq \pv$, corresponds in the representation space to the subspace inclusion $\mathcal{S}_p \subseteq \mathcal{S}_q$. This is a more general version of the orthogonality of disjoint classes of MMCR and MCR$^2$.

\begin{lemma}[Minterm orthogonality]
    \label{lemma:orthogonal_vs}
    Let $\Yv\in\{0,1\}^{c\times n}$ be a rank-c matrix with SVD $\Uv_Y\mathbf{\Sigma}_Y \Vv_Y^\top$, where $\Uv_Y\in O(c)$, $\mathbf{\Sigma}_Y \in \mathbb{R}^{c\times c}$ and $\Vv_Y \in\mathrm{St}_c(\mathbb{R}^n)\subset\mathbb{R}^{n\times c}$, with rows $\{\vv_i^\top\}_{i\in [n]}$. Let $\mathcal{I}$ be the largest index set such that for $i,j\in\mathcal{I}$ $\yv_i\neq \yv_j$. If $\mathrm{rank}\;\Yv = |\mathcal{I}|$ then $\{\vv_i\}_{i\in\mathcal{I}}$ is an orthogonal basis for $\mathbb{R}^c$.
\end{lemma}

Recall from Theorem \ref{theorem:min_nuclear_norm_ell2_squared}, that the minimizer of (\ref{eq:the_loss}) is precisely $\Vv_Y^\top$, up to a global scale and an orthogonal transformation. Combining this with Lemma \ref{lemma:orthogonal_vs}, we have that the normalized embeddings corresponding to the unique minterms $\{\xv_i/\|\xv_i\|_2\}_{i\in\mathcal{I}}$ form an orthonormal basis for the representation space $\mathbb{R}^c$. We will henceforth write this basis as $\{\ev_i\}_{i\in\mathcal{I}}$. Given a unit $\ell_2$-norm representation $\xv$, we have $\sum_{i\in\mathcal{I}} \langle \xv, \ev_i \rangle^2 = 1$ and we can define the probability of the $i$-th minterm $\yv_i$, given the image, as $P(\yv_i | \xv) := \langle \xv, \ev_i \rangle^2$ (by abuse of notation we will use the same symbols for a proposition and the corresponding random variable). Since $0\leq \langle \xv, \ev_i \rangle^2\leq 1$, the representations induce a categorical distribution over the dictionary $\{\yv_i\}_{i\in\mathcal{I}}$. More general propositions $\qv$ over the $c$ labels can be written as a disjunction of the conjunctions  $\{\yv_i\}_{i\in\mathcal{I}}$ \textit{i.e.}, in the disjunctive normal form (DNF).


\begin{lemma}[Propositions as projections]
    Let $\Yv\in\{0,1\}^{c\times n}$ verify the conditions of Lemma \ref{lemma:orthogonal_vs}. Given $\xv\in\mathbb{R}^d$, with $\|\xv\|_2 = 1$, define the posterior probability of the Bernoulli variable associated with the minterm $\yv_i$ as $P(\yv_i | \xv) := \langle \xv, \ev_i \rangle^2$. Then, $\forall\qv$ such that $\qv = \bigvee_{i\in\mathcal{J}} \yv_i$, for some $\mathcal{J}\subseteq\mathcal{I}$, there is a unique projection operator $\Pv_q$ such that $P(\qv|\xv) = \langle \xv, \Pv_q \xv \rangle$.
    \label{lemma:projections_are_propositions}
\end{lemma}

If given $\xv$, proposition $\qv$ holds, $\xv$ lies in the corresponding subspace \textit{i.e.}, $\Pv_q \xv =\xv$ and $P(\qv|\xv)=1$. Conversely, if $\qv$ is false then $\Pv_q \xv = 0$ and $P(\qv|\xv)=0$. The angle between the proposition subspace and the representation determines the probability $P(\qv|\xv)=\langle \xv, \Pv_q \xv \rangle$. Therefore, this formulation allows us to update representations via projections. If new information implies that $\qv$ holds for $\xv$, we consider the new embedding $\Pv_{q} \xv$. The same reasoning applies to disjunction, conjunction or negation of propositions, with $\Pv_{p \land q} = \Pv_{p} \Pv_{q}$, $\Pv_{p \lor q} = \Pv_{p}  + \Pv_{q} - \Pv_{p} \Pv_{q}$ and $\Pv_{\neg q} = \Iv - \Pv_{q}$. 

The set of subspaces of a vector space forms a lattice under subspace inclusion \cite{vonneumann}. Given subspaces $\mathcal{S}_p$ and $\mathcal{S}_q$, $\mathcal{S}_p \wedge \mathcal{S}_q$ is the set-theoretic intersection \textit{i.e.}, the greatest lower bound or \textit{meet}. The least upper bound or \textit{join}, $\mathcal{S}_p \lor \mathcal{S}_q$ is given by the closure of the sum $\mathcal{S}_p + \mathcal{S}_q$. From the one-to-one correspondence between projections and closed subspaces, the lattice structure has an algebraic characterization as $\mathcal{S}_p \leq \mathcal{S}_q \Leftrightarrow \Pv_p \leq \Pv_q \Leftrightarrow \Pv_p = \Pv_p \Pv_q$. If the projections commute then $\Pv_p \wedge \Pv_q = \Pv_p \Pv_q$ and $\Pv_p \lor \Pv_q = \Pv_p + \Pv_q - \Pv_p \Pv_q$. The \textit{largest} projection is the identity $\Iv$, and the smallest one is the zero operator $\mathbf{0}$. We have that $\Pv_p \lor \Pv_p^\perp = \Iv$, $\Pv_p \wedge \Pv_p^\perp = \mathbf{0}$ and $\Pv_p = \Pv_p^{\perp\perp}$, where $\Pv_p^\perp = \Iv - \Pv_p$ is called the orthocomplement of $\Pv_p$. Every subspace lattice is orthomodular \cite{Holland1975}, which is a weaker version of the distributive property of Boolean algebras. In our case, the subspace lattice is Boolean since all the projections share the eigenbasis $\{\ev_i\}_{i\in\mathcal{I}}$ and thus commute.

\begin{figure*}
\begin{subfigure}{0.24\linewidth}
    \includegraphics[trim={0.0cm 0.0cm 0.0cm 0},clip,width=\linewidth]{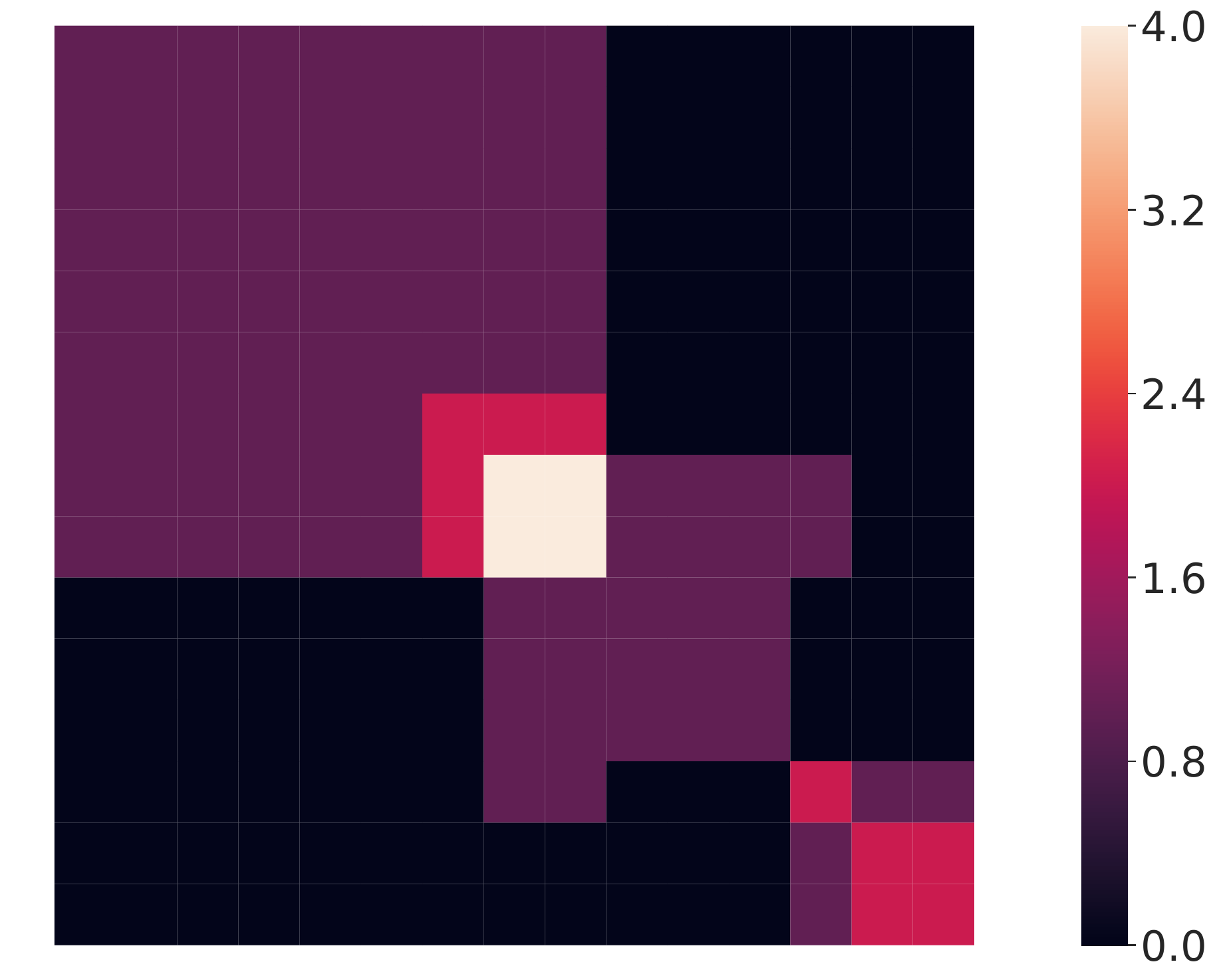}
    \caption{$\Yv^\top \Yv$}
    \label{fig:multilabel_gram_comparison:a}
\end{subfigure}
\hfill
\begin{subfigure}{0.24\linewidth}
    \includegraphics[width=\linewidth]{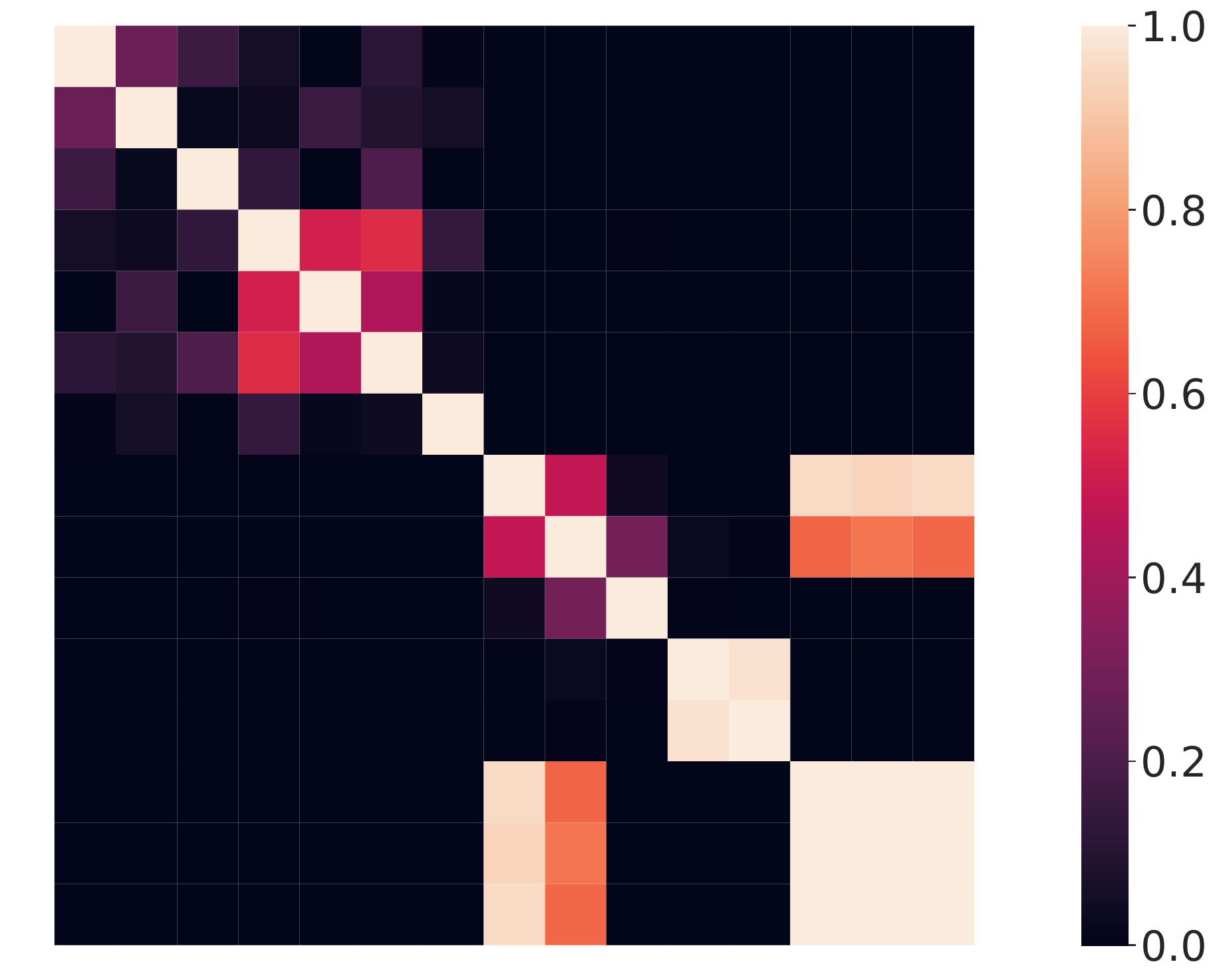}
    \caption{OLÉ}
     \label{fig:multilabel_gram_comparison:b}
\end{subfigure}
\hfill
\begin{subfigure}{0.24\linewidth}
    \includegraphics[width=\linewidth]{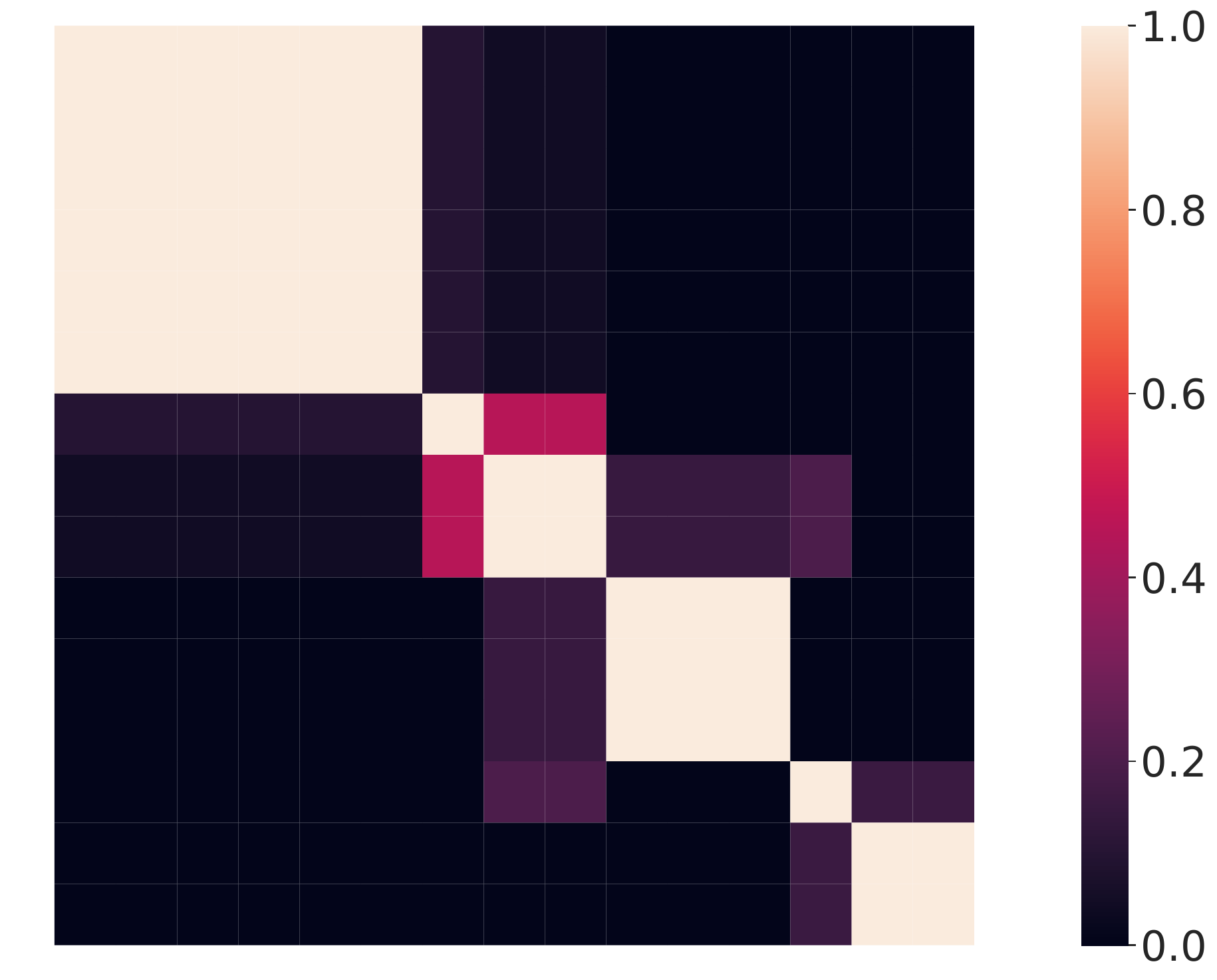}
    \caption{MMCR}
     \label{fig:multilabel_gram_comparison:c}
\end{subfigure}
\begin{subfigure}{0.24\linewidth}
    \includegraphics[width=\linewidth]{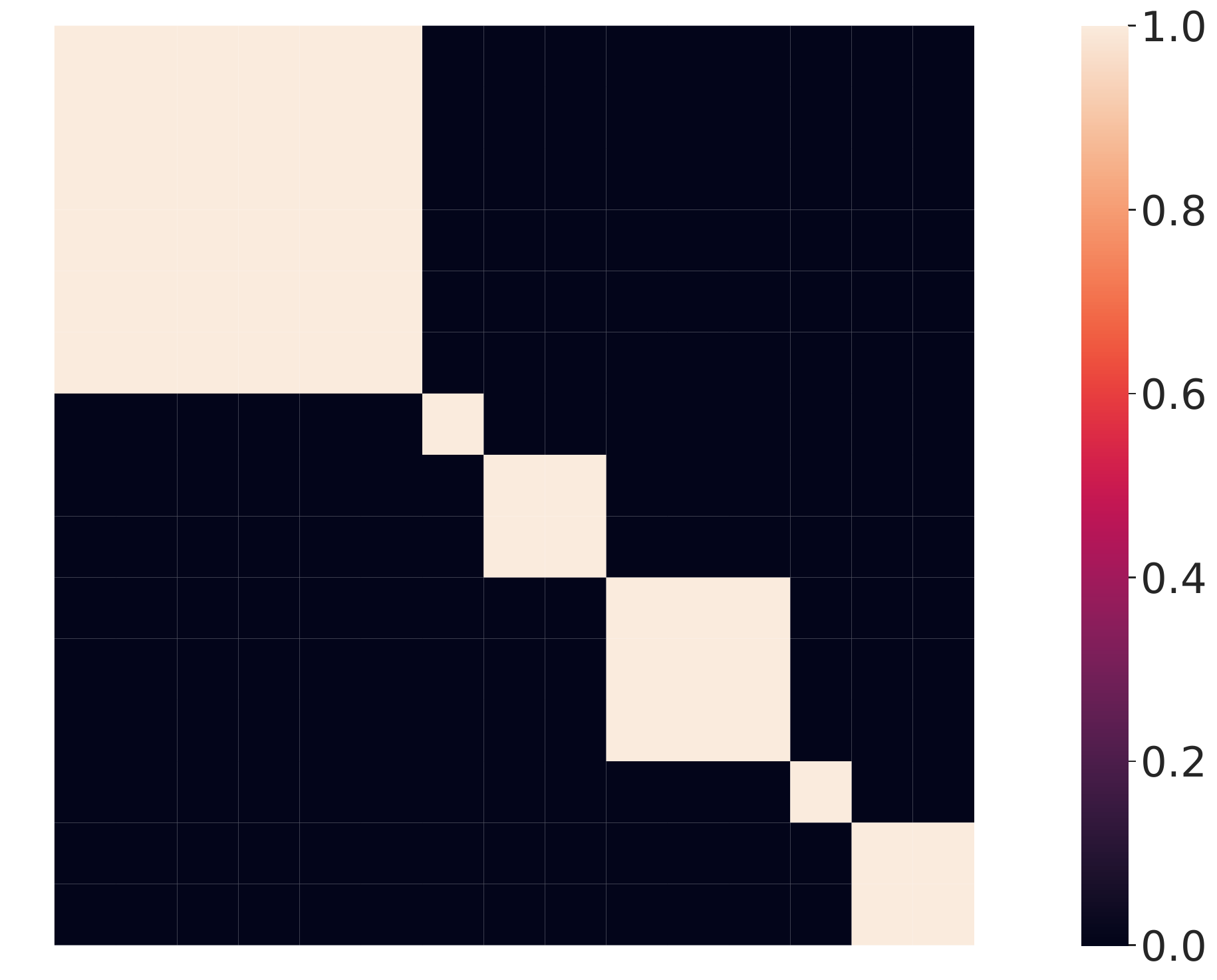}
    \caption{Ours}
     \label{fig:multilabel_gram_comparison:d}
\end{subfigure}%
\caption{Gram matrix of $\Yv \in \{0,1\}^{6\times 15}$ and of the representations optimized with OLÉ, MMCR and our loss.}
\label{fig:multilabel_gram_comparison}
\end{figure*}

\section{EXPERIMENTS}
\label{sec:experiments}
Given the fundamental differences in how representations are learned and inference conducted, it is important to verify that the unique properties of our methodology do not come at the expense of performance in canonical benchmarks. Thus, we first evaluate our approach in standard classification and then present results for retrieval from propositional queries. In both cases, we only train encoder models, not employing linear probes. At inference time, proposition probabilities are derived from the geometry of the representations (Lemma \ref{lemma:projections_are_propositions}). Our results show that the proposed loss can surpass traditional cross-entropy training in standard classification. In multi-label settings, our method captures the true semantic geometry, unlike MMCR or OLÉ, which allows for image retrieval from complex propositional queries.

\subsection{Implementation Details}
We implemented the loss (\ref{eq:the_loss}) in PyTorch, using the subdifferential of the nuclear norm \cite{recht2010guaranteed} $\Uv_X \Vv_X^\top \in \partial_{\Xv} \|\Xv\|_\ast$ as the descent direction in SGD. While the nuclear norm is non-smooth, we observed empirical convergence in the all the experiments conducted, akin to prior works. All experiments were performed on a Tesla T4 GPU with 16GB of memory. During training, the batched output of the image encoder is plugged directly in (\ref{eq:the_loss}), without any normalization or centering. The subspace representation of each proposition was computed from the training data, via the SVD of the embedding matrix of images for which the proposition holds. If $\Xv$ is the matrix of all the representations from the training set that verify $\qv$, then $\qv$ is represented by the subspace spanned by the singular vectors of $\Xv$ corresponding to the largest singular values. These are the only non-null singular values if (\ref{eq:the_loss}) attains its minimum. 

In order for the label matrix $\Yv$ to have full row-rank equal to the number of minterms, as required by Theorem \ref{theorem:min_nuclear_norm_ell2_squared}, we construct a minterm batch sampler so that every batch has as many minterms as labels, with each minterm having the same number of samples in the batch.

\subsection{Synthetic Experiments}
\label{sec:synthetic_experiments}
In order to shed light on the differences between our approach and the nuclear norm-based losses from OLÉ and MMCR, we optimized representations $\Xv$ for synthetic binary label matrices $\Yv$. Since all methods, ours included, provide theoretical guarantees for the orthogonality of disjoint classes, we focus here on the multi-label case. In Fig. \ref{fig:multilabel_gram_comparison} we plot the Gram matrix of a label matrix $\Yv^\top \Yv \in \{0,1\}^{n\times n}$, and the Gram matrices of the $\ell_2$-normalized solutions $\Xv^\top \Xv \in \mathbb{R}^{n\times n}$. OLÉ, MMCR were implemented as described in the respective papers, with MMCR requiring embedding normalization and OLÉ using a lower bound on the nuclear norm to avoid collapse during training. The hyperparameters of our loss (\ref{eq:the_loss}) were set to $\alpha=0.99$ and $\beta=0.7$, with optimization performed via gradient descent. In Fig. \ref{fig:multilabel_gram_comparison:d}, we observe that our representations verify $\forall\yv_i\neq \yv_j \implies \langle \xv_i, \xv_j\rangle = 0$ \textit{i.e.}, different minterms correspond to orthogonal directions, as consequence of Theorem \ref{theorem:min_nuclear_norm_ell2_squared} and Lemma \ref{lemma:orthogonal_vs}. Thus, our approach generalizes the disjoint classes setting, unlike the others methods considered, which converge to representations with no clear interpretation. In Appendix \ref{sec:app_synthetic_experiments}, we present additional experiments, showing that in standard classification, all three methods represent disjoint classes as orthogonal subspaces, as predicted theoretically. We also plot the convergence of the singular values of the our representations during training, which are in accordance with Theorem \ref{theorem:min_nuclear_norm_ell2_squared}.

\subsection{Standard Classification}
We report the classification performance of our method on MNIST \cite{deng2012mnist}, FashionMNIST \cite{xiao2017fashion}, CIFAR-10 and CIFAR-100 \cite{krizhevsky2009learning}, comparing it with the standard approach of using linear classifiers and optimizing a cross-entropy loss. 

We conducted experiments with two backbones, a standard ConvNet with 2 convolutional layers and a ResNet-18 \cite{he2016deep}, both with leaky ReLU activations and a fully connected output layer with dimension equal to the number of classes. Minimization of (\ref{eq:the_loss}) was carried out via SGD, with a batch size of 512 and no weight decay. For CIFAR-10 and CIFAR-100, we applied the standard data augmentations, including random cropping with a 4-pixel padding resized to $32\times 32$ and random horizontal flipping. Full training details can be consulted in Appendix \ref{sec:app_classification}. We report the results in Table \ref{tab:standard_classification}, which show that we perform on par with, or better than, the \textit{de facto} classification approach of using linear probes and optimizing a cross-entropy loss. To provide further insight into the geometry of the representation space, we plot in Fig. \ref{fig:inner_products_results} the squared inner products between the 1-dimensional subspaces of each class, and the squared inner products between the $\ell_2$-normalized embeddings of the test set, for three of the datasets. Notice that, in each dataset, the set of 1-d subspaces representing the different classes forms an orthonormal basis for the corresponding representation space, as predicted theoretically.

\begin{table}[]
    \centering
    \caption{Classification accuracy}
    \begin{tabular}{cccc}
    \toprule
        Dataset & Backbone & Ours & CrossEnt.\\
        \midrule
        \multirow{2}{*}{MNIST}         & ConvNet   & 0.992 & 0.992 \\
                                       & ResNet-18 & \textbf{0.996} & 0.995 \\
        \midrule
        \multirow{2}{*}{FashionMNIST} & ConvNet & \textbf{0.932}   & 0.930 \\
                                       & ResNet-18 & 0.935  & 0.935 \\
        \midrule
        CIFAR-10                       & ResNet-18 & \textbf{0.934} & 0.929\\
        \midrule
        CIFAR-100                      & ResNet-18  & \textbf{0.728} & 0.705 \\

    \bottomrule
    \end{tabular}
    \label{tab:standard_classification}
\end{table}

\begin{figure}
    \centering
    \begin{subfigure}[t]{0.32\linewidth}
    \includegraphics[width=\linewidth]{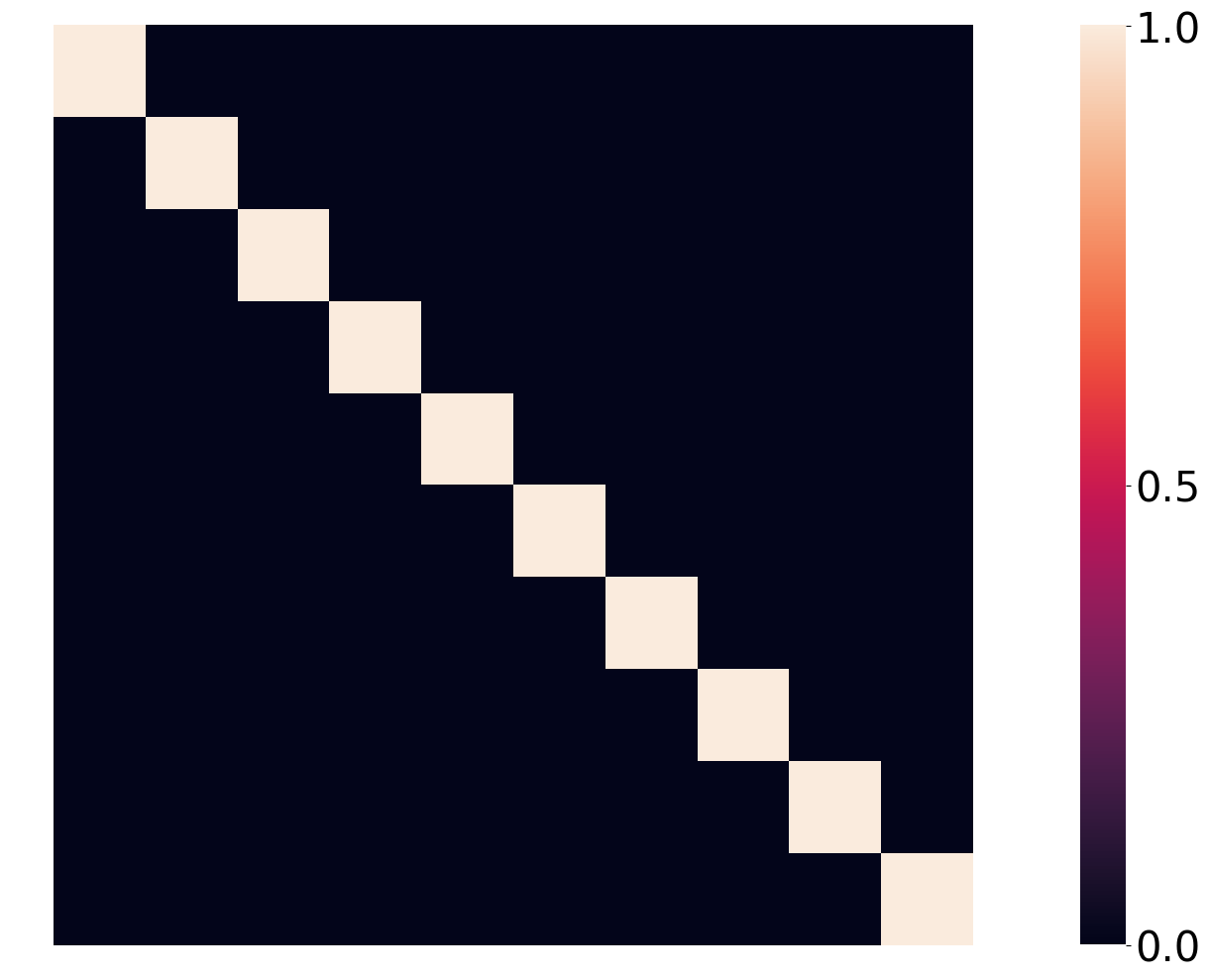}
      \caption{FashionMNIST}
      \label{fig:fashionmnist_classes_inner}
    \end{subfigure}
    \hfill
    \begin{subfigure}[t]{0.32\linewidth}
    \includegraphics[width=\linewidth]{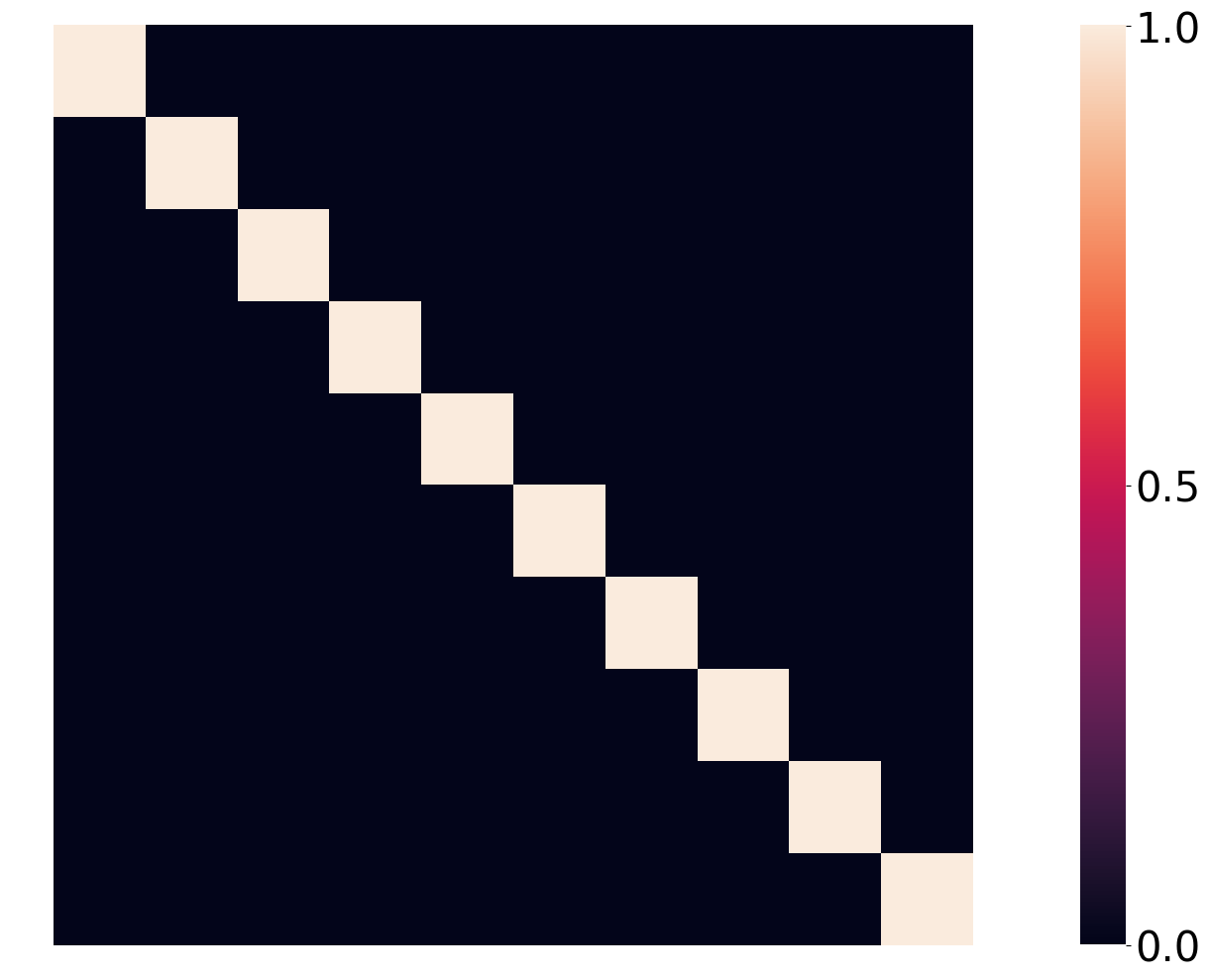}
      \caption{CIFAR-10}
    \label{fig:cifar10_minterms_inner}
    \end{subfigure}%
    \hfill
    \begin{subfigure}[t]{0.32\linewidth}
    \includegraphics[width=\linewidth]{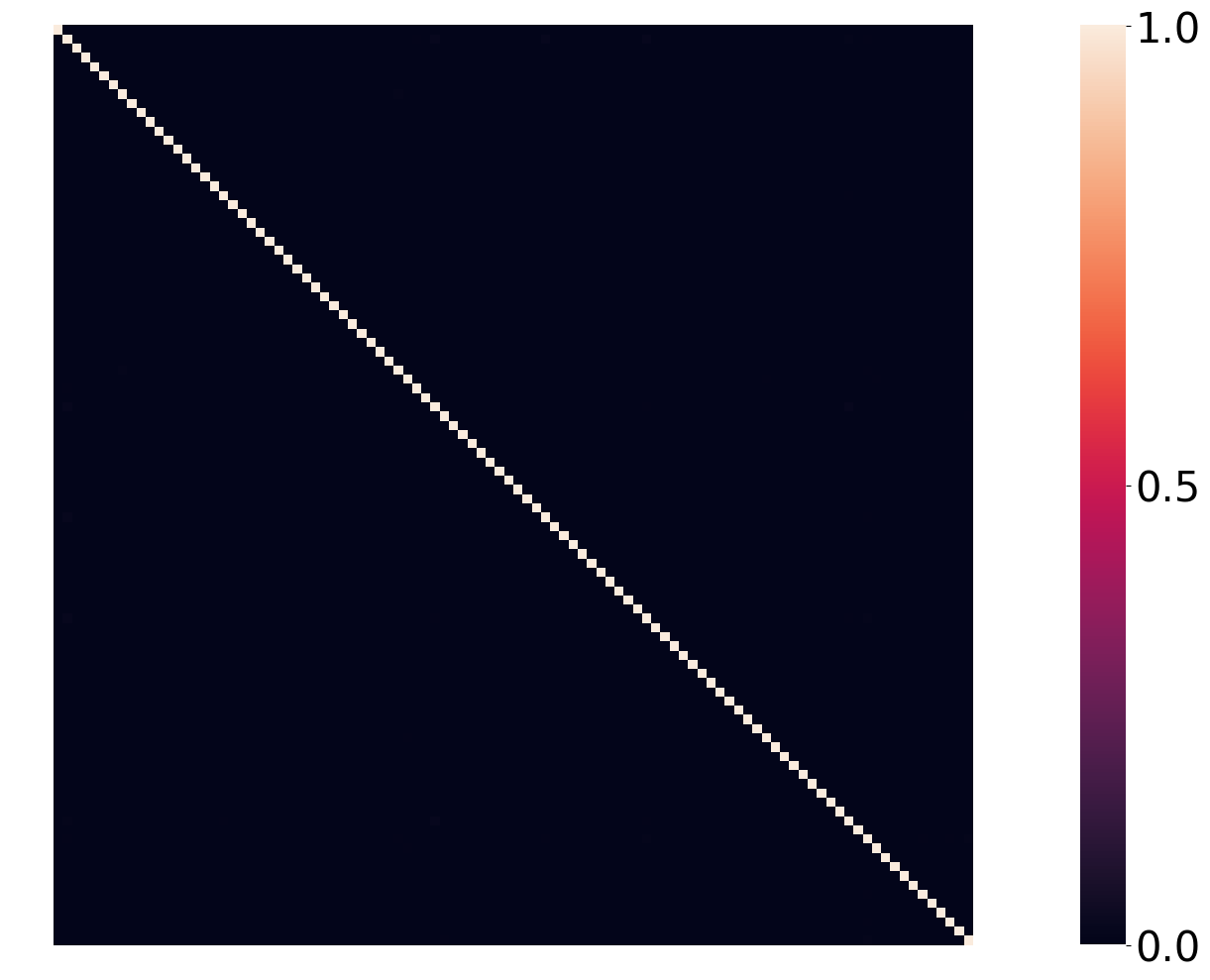}
      \caption{CIFAR-100}
    \label{fig:cifar100_minterms_inner}
    \end{subfigure}%

    \medskip  %
    \begin{subfigure}[t]{0.32\linewidth}
    \includegraphics[width=\linewidth]{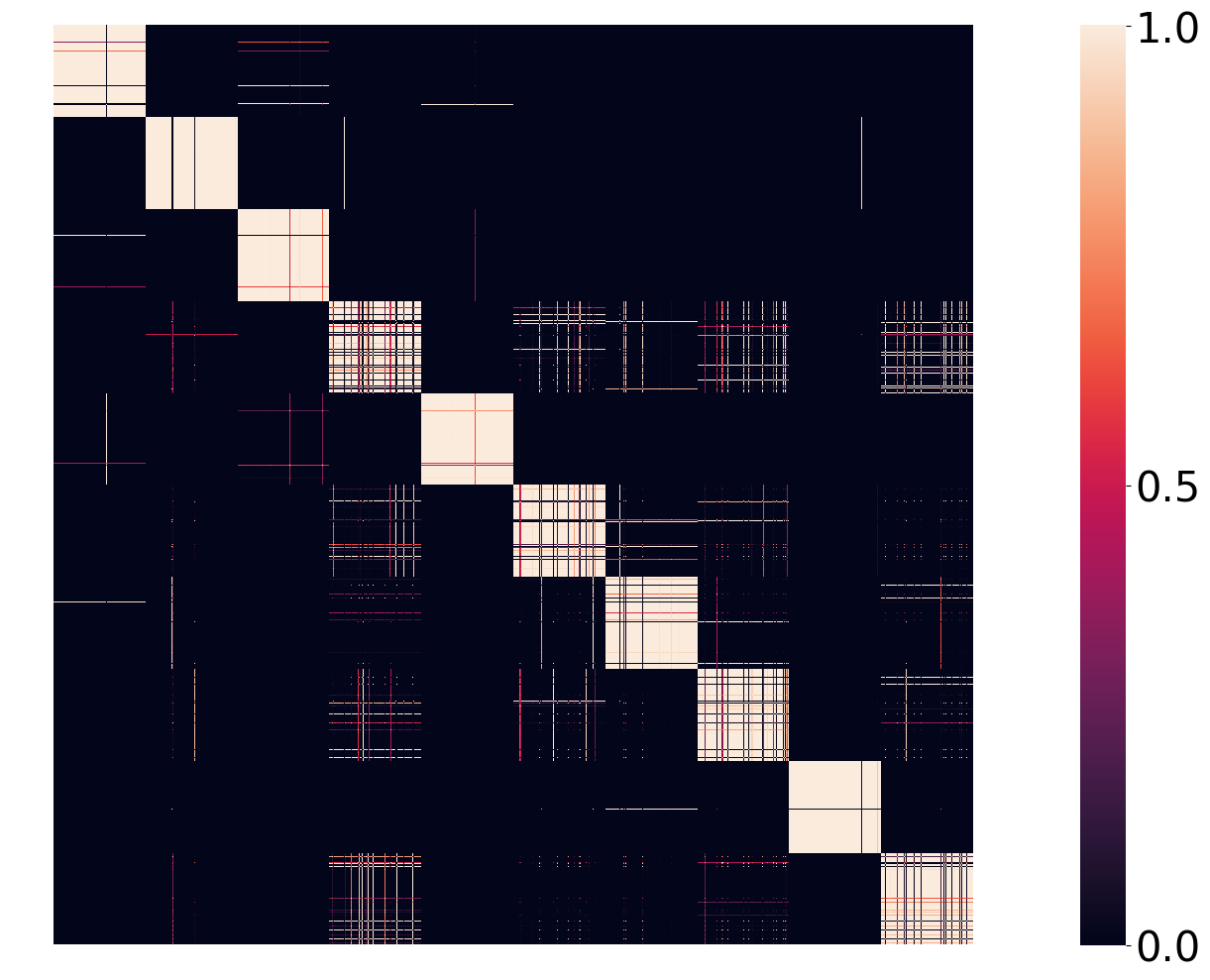}
      \caption{FashionMNIST}
      \label{fig:fashionmnist_classes_inner}
    \end{subfigure}
    \hfill
    \begin{subfigure}[t]{0.32\linewidth}
    \includegraphics[width=\linewidth]{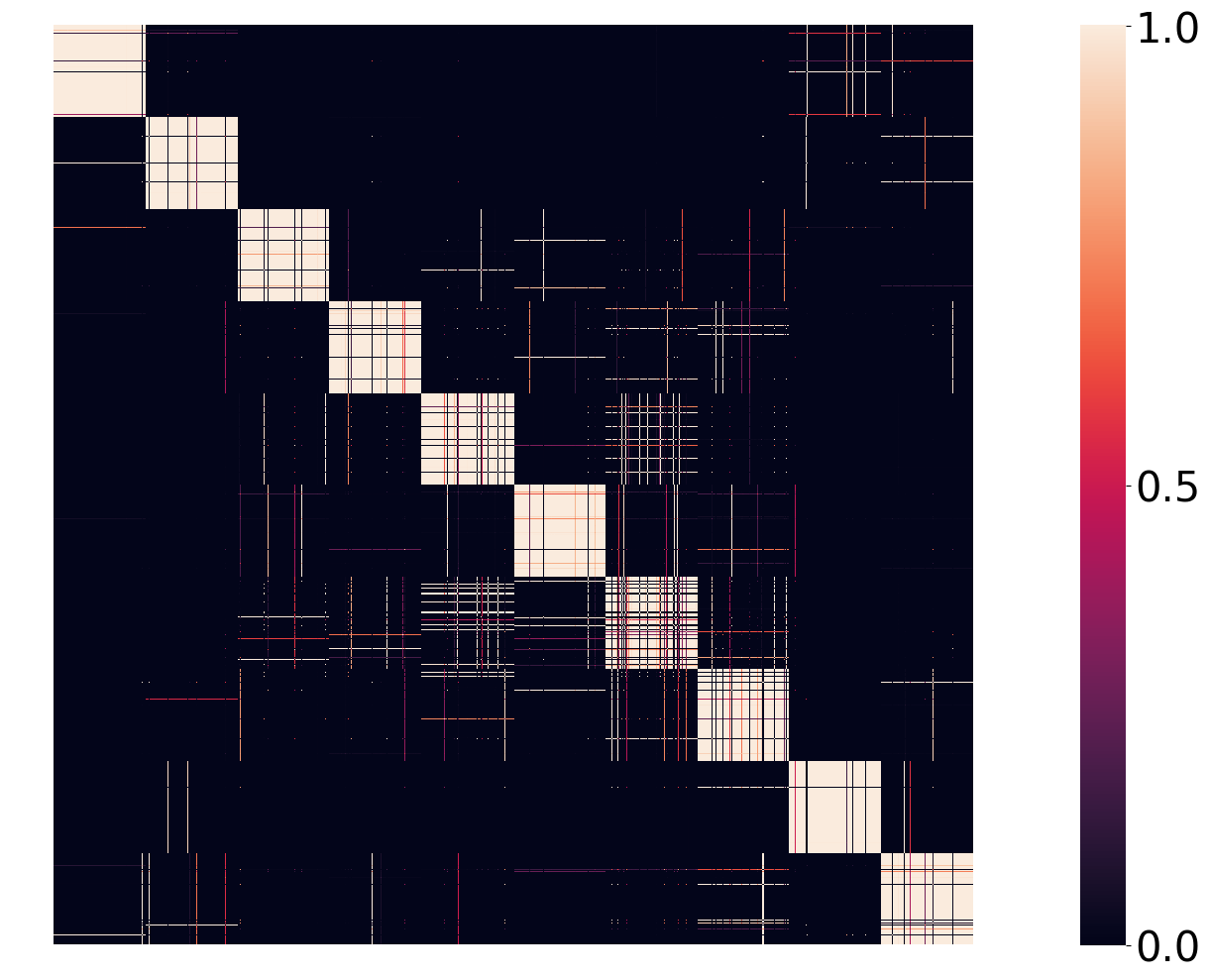}
      \caption{CIFAR-10}
    \label{fig:cifar10_minterms_inner}
    \end{subfigure}%
    \hfill
    \begin{subfigure}[t]{0.32\linewidth}
    \includegraphics[width=\linewidth]{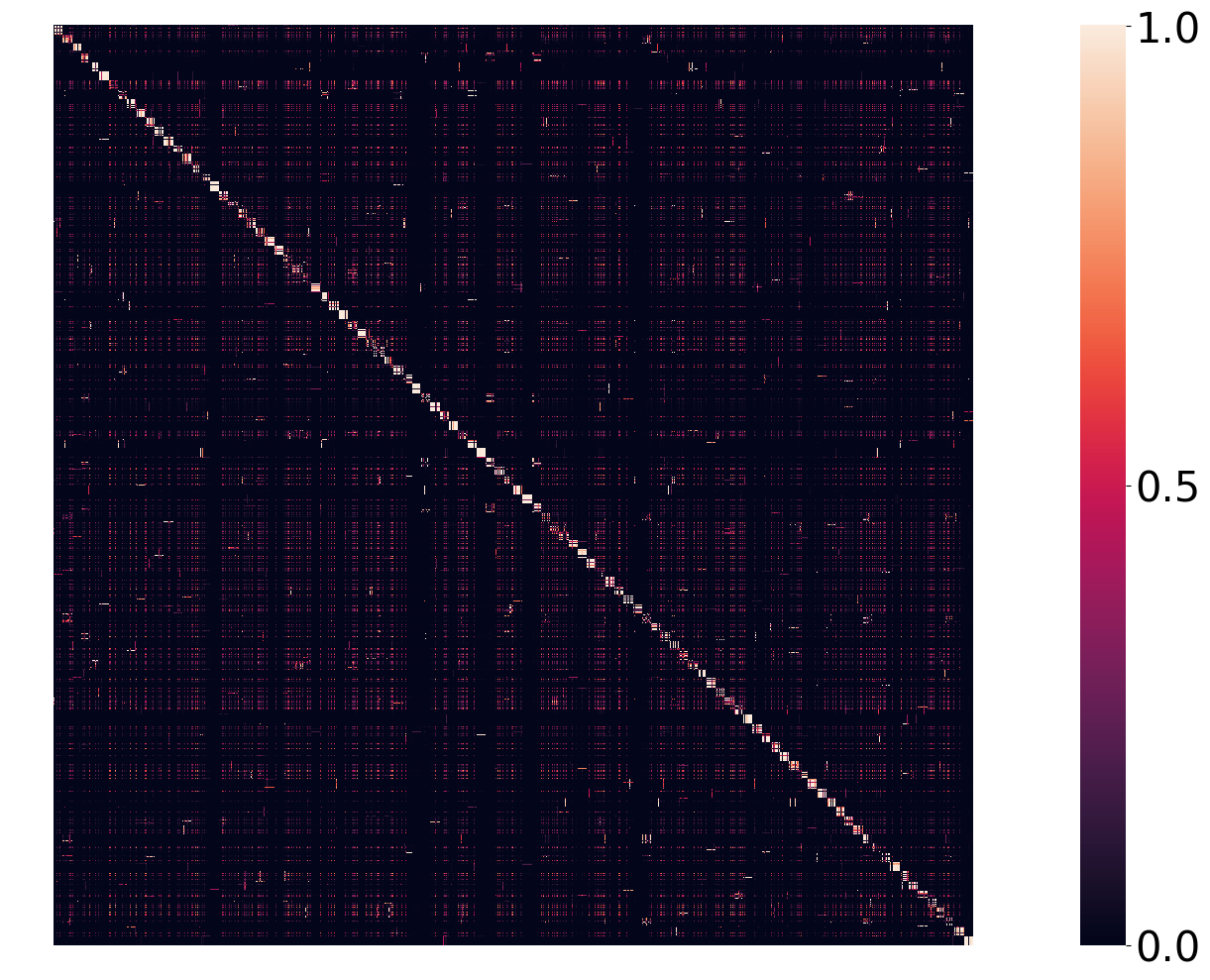}
      \caption{CIFAR-100}
    \label{fig:cifar100_minterms_inner}
    \end{subfigure}%
    \caption{\textbf{Top}: Inner products between the principal direction of each class. \textbf{Bottom}: Inner products between the unit $\ell_2$-norm test set embeddings.}
    \label{fig:inner_products_results}
\end{figure}


\begin{figure}
    \begin{subfigure}{\linewidth}
        \includegraphics[width=\linewidth]{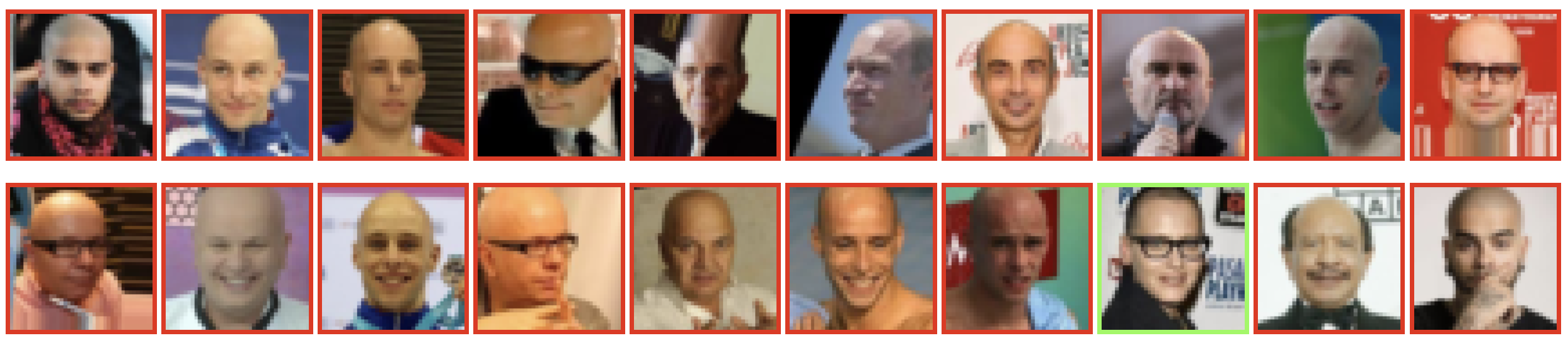}
        \caption{\textbf{CLIP}: \textit{A male person that is \textcolor{blue}{not bald}}}
        \label{fig:celeba_a}
    \end{subfigure}
    \medskip
    \begin{subfigure}{\linewidth}
        \includegraphics[width=\linewidth]{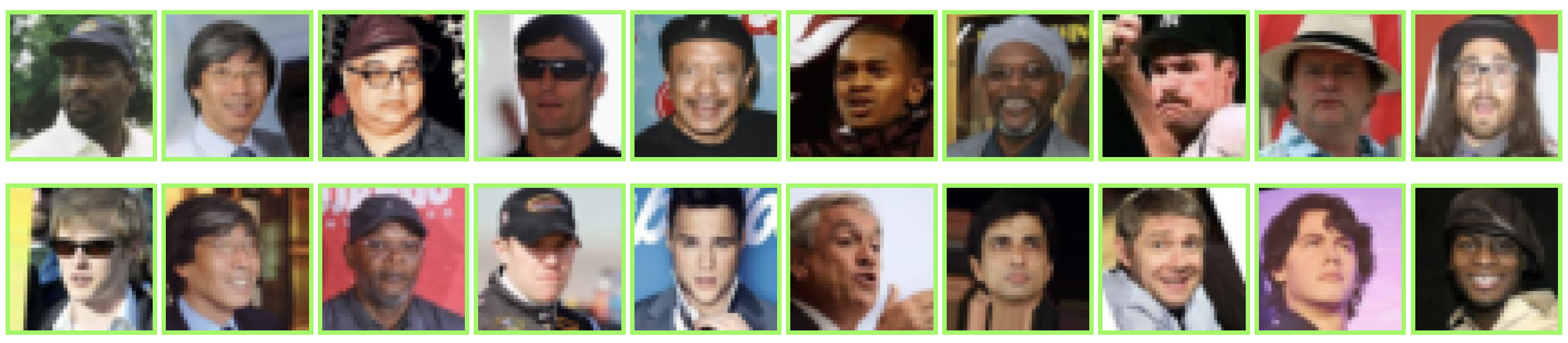}
        \caption{\textbf{Ours}: \texttt{Male} $\land\neg$ \texttt{Bald}}
        \label{fig:celeba_a}
    \end{subfigure}
    \medskip
    \begin{subfigure}{\linewidth}
        \includegraphics[width=\linewidth]{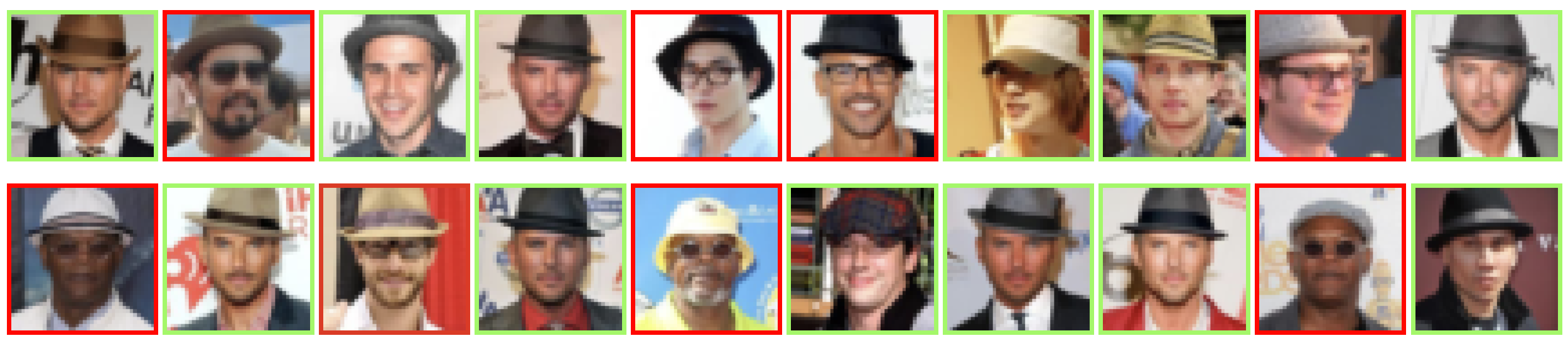}
        \caption{\textbf{CLIP}: \textit{A male person \textcolor{blue}{without eyeglasses} and wearing a hat}}
        \label{fig:celeba_a}
    \end{subfigure}
    \medskip
    \begin{subfigure}{\linewidth}
        \includegraphics[width=\linewidth]{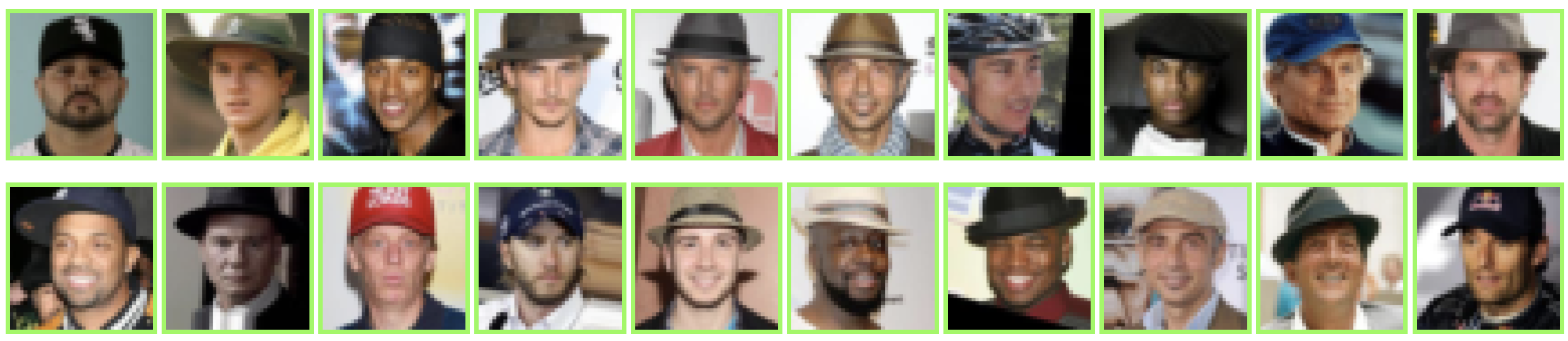}
        \caption{\textbf{Ours}: \texttt{Male} $\land$ \texttt{Wearing Hat} $\land\neg$ \texttt{Eyeglasses}}
        \label{fig:celeba_a}
    \end{subfigure}
    \medskip
    \begin{subfigure}{\linewidth}
        \includegraphics[width=\linewidth]{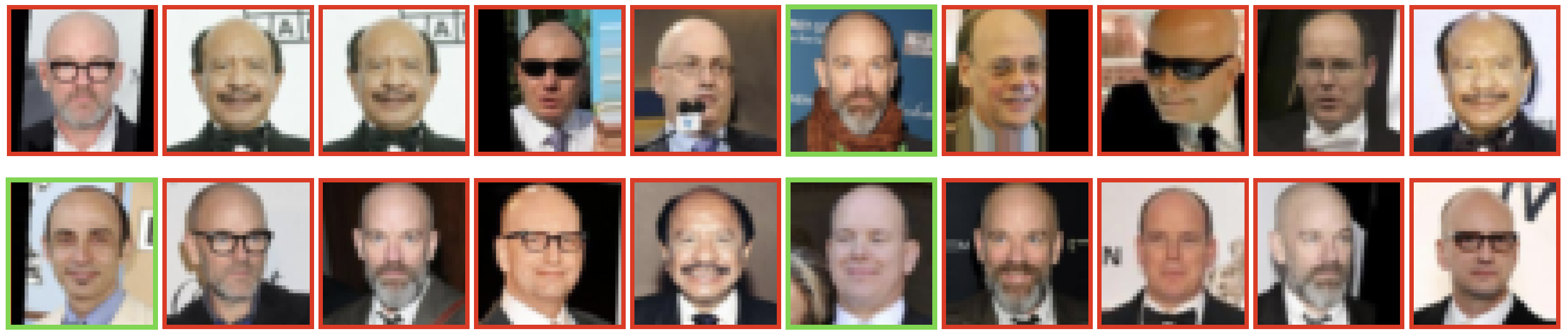}
        \caption{\textbf{CLIP}: \textit{A bald male person \textcolor{blue}{without eyeglasses and not wearing a necktie}}}
        \label{fig:celeba_a}
    \end{subfigure}
    \medskip
    \begin{subfigure}{\linewidth}
        \includegraphics[width=\linewidth]{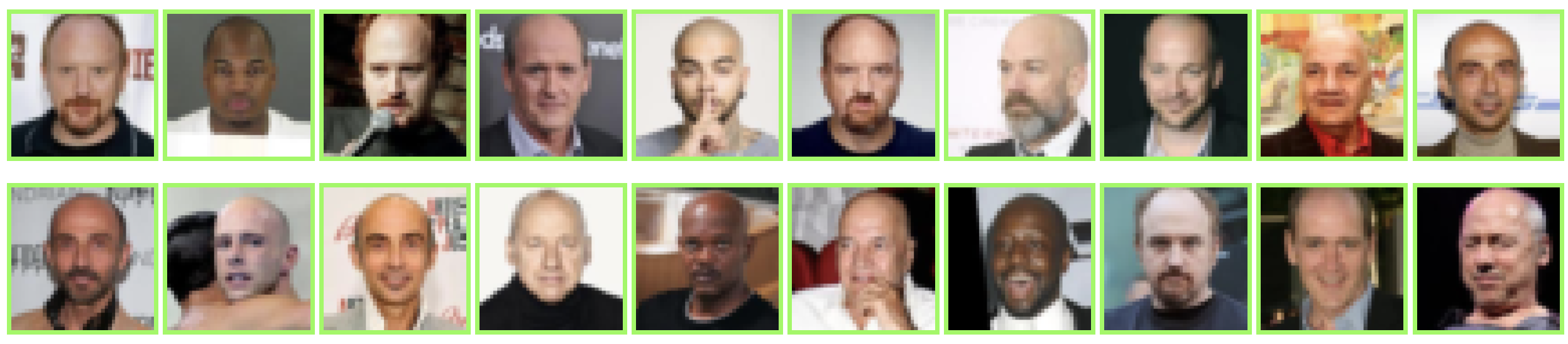}
        \caption{\textbf{Ours}: \texttt{Male} $\land$ \texttt{Bald} $\land\neg$ \texttt{Wearing Necktie} $\land\neg$ \texttt{Eyeglasses}}
        \label{fig:celeba_a}
    \end{subfigure}
    \medskip
    \begin{subfigure}{\linewidth}
        \includegraphics[width=\linewidth]{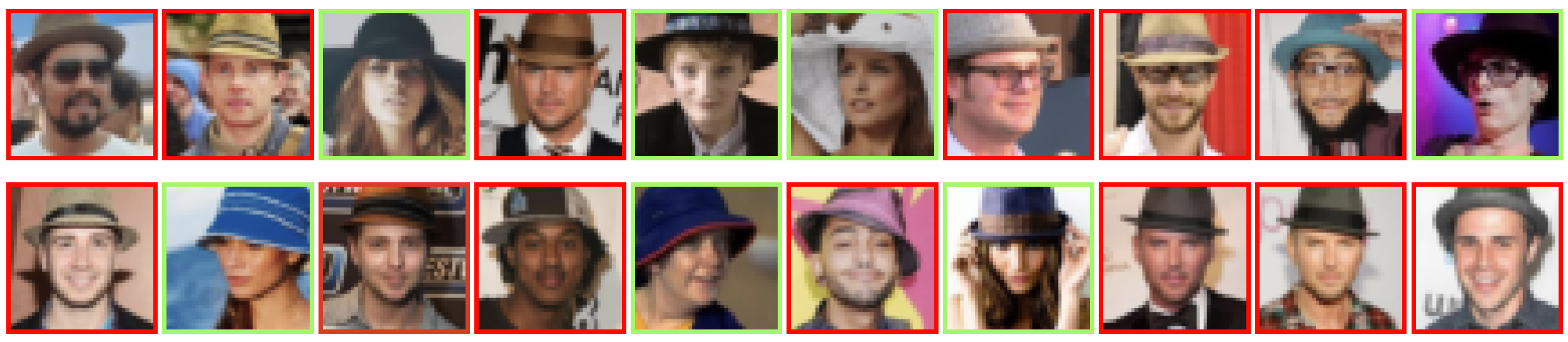}
        \caption{\textbf{CLIP}: \textit{A person that is \textcolor{blue}{not male} wearing a hat}}
        \label{fig:celeba_a}
    \end{subfigure}
    \medskip
    \begin{subfigure}{\linewidth}
        \includegraphics[width=\linewidth]{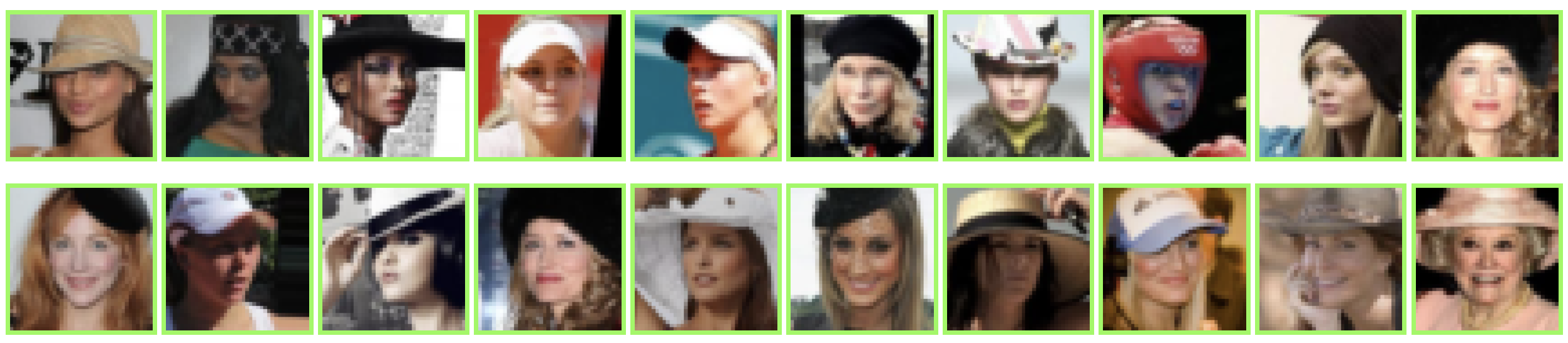}
        \caption{\textbf{Ours}: \texttt{Wearing Hat} $\land\neg$ \texttt{Male}}
        \label{fig:celeba_a}
    \end{subfigure}
    \caption{Ours vs CLIP's top-20 retrieved images from the test set of Celeb-A using propositional and the corresponding natural language queries.}
    \label{fig:celeba_retrieval}
\end{figure}

\subsection{Retrieval}
We showcase the applicability of our approach to the multi-label setting by using a subset of CelebA \cite{liu2015faceattributes} corresponding to five of the most consistent labels: $\texttt{Men}$, $\texttt{Bald}$, $\texttt{Eyeglasses}$, $\texttt{Wearing Hat}$, and $\texttt{Wearing Necktie}$. This subset comprises 72,900 training images, 8,952 validation images, and 8,303 test images, with 25 minterms. We used a ResNet-18 \cite{he2016deep} backbone with output dimension 25, optimized via SGD, with momentum set to $0.98$ and an initial learning rate of $10^{-4}$, decayed by 0.5 after 100 epochs. We used $\alpha=0.999$ and $\beta=0.01$. 

To evaluate the retrieval performance of our approach, we constructed propositional queries using subsets of the 5 labels as literals, along with the corresponding natural language counterparts. For example, the proposition with 3 literals $\texttt{Bald}\land\texttt{Male}\land\neg \texttt{Eyeglasses}$ translates to the query `\textit{a bald male person without eyeglasses}'. More examples are available in Appendix \ref{sec:app_retrieval}. In total, there are 242 queries, of which 171 have at least 10 examples in the test set. With these logical queries and their natural language counterparts, we compared our retrieval results against zero-shot (ZS) CLIP \cite{radford2021learning}, using the ViT-B/32 visual encoder. Motivated by \cite{yuksekgonul2022and}, we also present results for Bag-of-Words CLIP (BoW), where we prompt CLIP with an enumeration of the labels, without logical connectives. In the aforementioned example, the BoW query is simply `\textit{bald male eyeglasses}'. We report the mean average precision (mAP) and Precision@10 (Pr@10) in Table \ref{tab:retrieval}. Given prior observations that vision-language models struggle with negation, we present results separately for positive queries (no negated literals) and queries containing at least one negated literal. 

Our method consistently outperforms CLIP, with the performance gap widening for queries involving negation. As expected, CLIP performs similarly for natural language and BoW in the positive case. For queries with negated literals, removing logical connectives does not have a drastic impact on CLIP’s performance, suggesting that CLIP represents these queries in a manner that is closer to a bag-of-words, rather than a logically meaningful embedding. We present retrieval examples for multiple queries in Fig. \ref{fig:celeba_retrieval}, showing the top-20 retrieved images by our model and by CLIP. Green and red image borders indicate true and false positives, respectively.

\begin{table}[]
    \centering
    \caption{Celeb-A retrieval results}
    \begin{tabular}{ccccc}
    \toprule 
        \multirow{2}{*}{Model} & \multicolumn{2}{c}{Positive} & \multicolumn{2}{c}{w/ negations}\\
              & Pr@10 & mAP & Pr@10 & mAP \\
        \midrule
        Ours  & \textbf{0.93} & \textbf{0.75} & \textbf{0.88} & \textbf{0.79} \\
        ZS-CLIP  & 0.67 & 0.46 & 0.21 & 0.30 \\
        ZS-CLIP (BoW) & 0.66 & 0.48 & 0.11 & 0.24 \\
    \bottomrule
    \end{tabular}
    \label{tab:retrieval}
\end{table}

\section{CONCLUSION}
In this work, we introduced a novel approach for learning image representations that adhere to the semantic geometry of the data. Motivated by recent advances in self-supervised learning, our method is based on a new nuclear norm-based loss function, which we show to yield the spectral embeddings of the data. Our approach, being theoretically grounded, diverges both from contrastive learning approaches and methods that simply emphasize class orthogonality. Instead, we show that our representations form a subspace Boolean lattice, which facilitates the definition of probabilistic propositional queries through projection operators. We believe that our work can thus reveal new possibilities for handling complex retrieval tasks and multi-label classification. Finally, while we presented our approach targeting visual-semantic representations, the proposed loss is modality agnostic.   

\paragraph{Limitations and Future Work} While we provide theoretical guarantees for our approach, some limitations are worthy of mention. Firstly, the use of the nuclear norm leads to the non-smoothness of the loss function. This can make optimization not as straightforward as that of other losses, namely the cross-entropy or the log-determinant. Nevertheless, all the experiments conducted point to good convergence properties when employing the sub-differential as the descent direction. Secondly, and akin to similar works in representation learning, we did not investigate the convergence properties of the batched version of the proposed nuclear norm-based loss, relying uniquely on empirical validation. Finally, we consider our main contribution to be theoretical and additional experiments would showcase the full potential of the proposed methodology in multi-label classification and retrieval. 

\subsubsection*{Acknowledgements}
This work was supported by LARSyS funding (DOI: \nolinkurl{10.54499/LA/P/0083/2020}, \nolinkurl{10.54499/UIDP/50009/2020}, and \nolinkurl{10.54499/UIDB/50009/2020}), through Fundação para a Ciência e a Tecnologia. G Moreira was also supported via grant \nolinkurl{SFRH/BD/151466/2021} through the Carnegie Mellon Portugal program.  M Marques and J Costeira were also supported by the PT Smart Retail project (PRR - \nolinkurl{02/C05-i11/2024.C645440011-00000062}), through IAPMEI - Agência para a Competitividade e Inovação.

\bibliographystyle{apalike}
\bibliography{references}

\section*{Checklist}
 \begin{enumerate}

 \item For all models and algorithms presented, check if you include:
 \begin{enumerate}
   \item A clear description of the mathematical setting, assumptions, algorithm, and/or model. [Yes - Sections \ref{sec:theoretical_results}]
   \item An analysis of the properties and complexity (time, space, sample size) of any algorithm. [No]
   \item (Optional) Anonymized source code, with specification of all dependencies, including external libraries. [Yes - \url{https://github.com/gabmoreira/subspaces}]
 \end{enumerate}

 \item For any theoretical claim, check if you include:
 \begin{enumerate}
   \item Statements of the full set of assumptions of all theoretical results. [Yes - Assumptions presented in all Lemmas and Theorems.]
   \item Complete proofs of all theoretical results. [Yes - Proofs in Appendix \ref{sec:app_proofs}] 
   \item Clear explanations of any assumptions. [Yes]     
 \end{enumerate}

 \item For all figures and tables that present empirical results, check if you include:
 \begin{enumerate}
   \item The code, data, and instructions needed to reproduce the main experimental results (either in the supplemental material or as a URL). [Yes - \url{https://github.com/gabmoreira/subspaces}.]
   \item All the training details (e.g., data splits, hyperparameters, how they were chosen). [Yes - Section \ref{sec:experiments} and Appendix \ref{sec:app_experiments}] 
         \item A clear definition of the specific measure or statistics and error bars (e.g., with respect to the random seed after running experiments multiple times). [No]
         \item A description of the computing infrastructure used. (e.g., type of GPUs, internal cluster, or cloud provider). [Yes - Section \ref{sec:experiments}]
 \end{enumerate}

 \item If you are using existing assets (e.g., code, data, models) or curating/releasing new assets, check if you include:
 \begin{enumerate}
   \item Citations of the creator If your work uses existing assets. [Yes - MNIST, FashionMNIST, CIFAR-10, CIFAR-100 and CelebA] 
   \item The license information of the assets, if applicable. [Yes]
   \item New assets either in the supplemental material or as a URL, if applicable. [Not Applicable]
   \item Information about consent from data providers/curators. [Not Applicable - Publicly available data] 
   \item Discussion of sensible content if applicable, e.g., personally identifiable information or offensive content. [Not Applicable]
 \end{enumerate}

 \item If you used crowdsourcing or conducted research with human subjects, check if you include:
 \begin{enumerate}
   \item The full text of instructions given to participants and screenshots. [Not Applicable]
   \item Descriptions of potential participant risks, with links to Institutional Review Board (IRB) approvals if applicable. [Not Applicable]
   \item The estimated hourly wage paid to participants and the total amount spent on participant compensation. [Not Applicable]
 \end{enumerate}

 \end{enumerate}

\clearpage
\newpage
\onecolumn
\aistatstitle{Learning Visual-Semantic Subspace Representations:\\
Supplementary Materials}
\appendix
\addcontentsline{toc}{section}{Appendix}
\input{appendix_a}

\clearpage
\newpage
\input{appendix_b}
\clearpage
\newpage
\input{appendix_c}

\end{document}

%% file: projections.tikz
\tikzset{every picture/.style={line width=0.3pt}} 

\begin{tikzpicture}[x=0.75pt,y=0.75pt,yscale=-1,xscale=1]

\draw    (286.93,136.4) -- (286.93,7.65) ;
\draw [shift={(286.93,5.65)}, rotate = 90] [color={rgb, 255:red, 0; green, 0; blue, 0 }  ][line width=0.75]    (10.93,-3.29) .. controls (6.95,-1.4) and (3.31,-0.3) .. (0,0) .. controls (3.31,0.3) and (6.95,1.4) .. (10.93,3.29)   ;
\draw    (286.93,136.4) -- (450.3,136.4) ;
\draw [shift={(452.3,136.4)}, rotate = 180] [color={rgb, 255:red, 0; green, 0; blue, 0 }  ][line width=0.75]    (10.93,-3.29) .. controls (6.95,-1.4) and (3.31,-0.3) .. (0,0) .. controls (3.31,0.3) and (6.95,1.4) .. (10.93,3.29)   ;
\draw    (286.93,136.4) -- (190.59,192.59) ;
\draw [shift={(188.87,193.6)}, rotate = 329.74] [color={rgb, 255:red, 0; green, 0; blue, 0 }  ][line width=0.75]    (10.93,-3.29) .. controls (6.95,-1.4) and (3.31,-0.3) .. (0,0) .. controls (3.31,0.3) and (6.95,1.4) .. (10.93,3.29)   ;
\draw [color={rgb, 255:red, 155; green, 155; blue, 155 }  ,draw opacity=1 ] [dash pattern={on 0.84pt off 2.51pt}]  (344.13,160.91) -- (344.13,62.85) ;
\draw [color={rgb, 255:red, 246; green, 174; blue, 45 }  ,draw opacity=1 ][line width=1.5]    (286.93,136.4) -- (344.13,160.91) ;
\draw [color={rgb, 255:red, 155; green, 155; blue, 155 }  ,draw opacity=1 ] [dash pattern={on 0.84pt off 2.51pt}]  (246.07,160.91) -- (344.13,160.91) ;
\draw [color={rgb, 255:red, 155; green, 155; blue, 155 }  ,draw opacity=1 ] [dash pattern={on 0.84pt off 2.51pt}]  (384.17,137.2) -- (344.13,160.91) ;
\draw [color={rgb, 255:red, 155; green, 155; blue, 155 }  ,draw opacity=1 ] [dash pattern={on 0.84pt off 2.51pt}]  (384.17,137.2) -- (384.17,39.14) ;
\draw  [dash pattern={on 0.84pt off 2.51pt}]  (384.17,39.14) -- (344.13,62.85) ;
\draw [color={rgb, 255:red, 155; green, 155; blue, 155 }  ,draw opacity=1 ] [dash pattern={on 0.84pt off 2.51pt}]  (246.07,160.91) -- (246.07,58.76) ;
\draw [color={rgb, 255:red, 134; green, 187; blue, 216 }  ,draw opacity=1 ][line width=1.5]    (286.93,136.4) -- (384.17,39.14) ;
\draw [color={rgb, 255:red, 155; green, 155; blue, 155 }  ,draw opacity=1 ] [dash pattern={on 0.84pt off 2.51pt}]  (245.94,62.89) -- (344.13,62.85) ;
\draw [color={rgb, 255:red, 117; green, 142; blue, 79 }  ,draw opacity=1 ][line width=1.5]    (286.93,136.4) -- (245.94,62.89) ;
\draw [color={rgb, 255:red, 155; green, 155; blue, 155 }  ,draw opacity=1 ] [dash pattern={on 0.84pt off 2.51pt}]  (286.93,38.33) -- (384.17,39.14) ;
\draw [color={rgb, 255:red, 51; green, 101; blue, 138 }  ,draw opacity=1 ][line width=1.5]    (286.93,136.4) -- (286.93,38.33) ;
\draw [color={rgb, 255:red, 2; green, 48; blue, 71 }  ,draw opacity=1 ][line width=1.5]    (384.17,136.4) -- (286.93,136.4) ;
\draw [color={rgb, 255:red, 242; green, 100; blue, 25 }  ,draw opacity=1 ][line width=1.5]    (246.07,160.91) -- (286.93,136.4) ;
\draw [color={rgb, 255:red, 155; green, 155; blue, 155 }  ,draw opacity=1 ] [dash pattern={on 0.84pt off 2.51pt}]  (286.93,38.33) -- (245.94,62.89) ;
\draw  [draw opacity=0][fill={rgb, 255:red, 246; green, 174; blue, 45 }  ,fill opacity=0.1 ] (281.76,136.5) -- (440.42,136.5) -- (354.45,190.73) -- (195.79,190.73) -- cycle ;
\draw  [draw opacity=0][fill={rgb, 255:red, 134; green, 187; blue, 216 }  ,fill opacity=0.1 ] (286.72,13.93) -- (439.3,13.93) -- (439.3,136.5) -- (286.72,136.5) -- cycle ;
\draw  [draw opacity=0][fill={rgb, 255:red, 117; green, 142; blue, 79 }  ,fill opacity=0.1 ] (195.97,68.1) -- (287.53,13.93) -- (287.49,135.07) -- (195.92,189.25) -- cycle ;
\draw [line width=1.5]    (286.93,136.4) -- (344.13,62.85) ;

\draw (420,118) node [anchor=north west][inner sep=0.75pt]    {$\mathbf{p} \land \mathbf{q}$};
\draw (212,178.75) node [anchor=north west][inner sep=0.75pt]    {$\mathbf{p} \land \neg \mathbf{q}$};
\draw (291.53,8.0) node [anchor=north west][inner sep=0.75pt]    {$\neg \mathbf{p} \land \mathbf{q}$};
\draw (307.4,167.19) node [anchor=north west][inner sep=0.75pt]  [color={rgb, 255:red, 208; green, 2; blue, 27 }  ,opacity=1 ,rotate=-357.33,xslant=1.72]  {$ \begin{array}{l}
\mathbf{\textcolor[rgb]{0.96,0.68,0.18}{p}}\\
\end{array}$};
\draw (418.68,61.24) node [anchor=north west][inner sep=0.75pt]  [color={rgb, 255:red, 134; green, 187; blue, 216 }  ,opacity=1 ]  {$\mathbf{q}$};
\draw (203.23,106.69) node [anchor=north west][inner sep=0.75pt]  [color={rgb, 255:red, 117; green, 142; blue, 79 }  ,opacity=1 ,rotate=-332.09,xslant=0.65]  {$\mathbf{p} \oplus \mathbf{q}$};
\draw (313,61.59) node [anchor=north west][inner sep=0.75pt]    {$ \begin{array}{l}
\mathbf{x}\\
\end{array}$};
\draw (356.1,155.12) node [anchor=north west][inner sep=0.75pt]  [color={rgb, 255:red, 208; green, 2; blue, 27 }  ,opacity=1 ,xslant=1.24]  {$\textcolor[rgb]{0.96,0.68,0.18}{P(\mathbf{p|x})}$};
\draw (369.24,19.29) node [anchor=north west][inner sep=0.75pt]  [color={rgb, 255:red, 134; green, 187; blue, 216 }  ,opacity=1 ]  {$P(\mathbf{q} |\mathbf{x})$};
\draw (209.37,62.77) node [anchor=north west][inner sep=0.75pt]  [color={rgb, 255:red, 117; green, 142; blue, 79 }  ,opacity=1 ,rotate=-326.78,xslant=0.66]  {$P(\mathbf{p} \oplus \mathbf{q} | \mathbf{x})$};
\draw (215.71,157.06) node [anchor=north west][inner sep=0.75pt]  [font=\footnotesize,color={rgb, 255:red, 246; green, 174; blue, 45 }  ,opacity=1 ,rotate=-329.78,xslant=0.59]  {$\textcolor[rgb]{0.95,0.39,0.1}{P(\mathbf{p} \land \neg \mathbf{q} |\mathbf{x})}$};
\draw (324.12,116.72) node [anchor=north west][inner sep=0.75pt]  [font=\footnotesize,color={rgb, 255:red, 74; green, 144; blue, 226 }  ,opacity=1 ]  {$\textcolor[rgb]{0.01,0.19,0.28}{P(\mathbf{p} \land \mathbf{q} |\mathbf{x})}$};
\draw (292.78,36.24) node [anchor=north west][inner sep=0.75pt]  [font=\footnotesize,color={rgb, 255:red, 51; green, 101; blue, 138 }  ,opacity=1 ]  {$\textcolor[rgb]{0.2,0.4,0.54}{P( \neg \mathbf{p} \land \mathbf{q} |\mathbf{x})}$};

\end{tikzpicture}

%% file: appendix_a.tex
\vspace{-3cm}
\section{Definitions}
\label{sec:app_theory}
We consider finite dimensional vector spaces over $\mathbb{R}$ with inner product $\langle \uv,\vv \rangle = \uv^\top \vv$ \textit{i.e.}, the usual dot product. The induced norm is thus $\|\xv\|_2 = \sqrt{\langle \xv, \xv\rangle}$. A linear operator is called symmetric (self-adjoint) if $\Tv=\Tv^\top$. A symmetric operator $\Tv$ is said to be positive if $\langle \Tv\xv, \xv\rangle>0, \forall \xv\in\mathbb{R}^d.$ A projection operator is an idempotent self-adjoint \textit{i.e.}, $\Tv=\Tv^\top=\Tv^2$. 

\begin{definition}[Spectral Norm] 
    The spectral norm is defined as 
    \begin{equation}
        \|\Tv\|_2 := \max_{\xv} \frac{\|\Tv \xv\|_2}{\|\xv\|_2}.
    \end{equation}
    If $\sigma_1\geq\dots\geq\sigma_d$ are singular values of $\Tv$ then $\|\Tv\|_2 = \sigma_1$.
\end{definition}

\begin{definition}[Nuclear Norm] 
    The nuclear norm of an operator, denoted $\|\cdot\|_\ast$, is the $\ell_1$-norm of its singular values. For $\Tv$ with singular values $\sigma_1\geq\dots\geq\sigma_n$, $\|\Tv\|_\ast = \sum_{i=1}^n \sigma_i$. It follows that $\|\Tv\|_\ast = \mathrm{Tr}\left(\left(\Tv^\top\Tv\right)^\frac{1}{2}\right)$. The nuclear norm is the convex envelope of the rank.
\end{definition}

\begin{definition}[Partially Ordered Set] 
    A partially ordered set (poset) is a set where a partial order is defined. Given a set $\mathcal{X}$ a (weak) partial order is a binary relation $\leq$  between certain pairs of elements that is reflexive (every element relates to itself) $a \leq a$, antisymmetric $a \leq b, b\leq a \Rightarrow a=b$ and transitive $a \leq b, b\leq c \Rightarrow a \leq c$. Strong partial orders are not reflexive. A function between posets is called order-preserving or isotone if $x\leq y \Rightarrow f(x) \leq f(y)$. Every poset $(\mathcal{X},\preceq)$ is isomorphic to the poset $(2^\mathcal{X},\subseteq)$ \textit{i.e.}, the poset defined on the power set of $\mathcal{X}$, with inclusion as the partial order.
\end{definition}

\begin{definition}[Transitive Relation]
    A relation $R$ is transitive if $x R y, y R z \implies z R z$. For example, inclusion and composition are transitive relations.
\end{definition}

\begin{definition}[Join and Meet]
    An element $m$ of a poset $\mathcal{P}$ is the meet of $x, y$, denoted $x \wedge y$, if $m\leq x, m\leq y$ and if $w \leq x, w\leq y$ then $w \leq m$. Thus, $m$ is the greatest lower bound (infimum). An element $j$  of $\mathcal{P}$ is the join of $x,y$ denoted $x\lor y$ if is is the lowest upper bound (supremum). 
\end{definition}

\begin{definition}[Order Lattice]
    A poset is a lattice if every two-element subset has a meet and a join.
\end{definition}

\begin{definition}[Modular Lattice]
    A lattice is called modular if for all elements a, b and c the implication
    \begin{equation}
        a \leq c \implies a \lor (b \wedge c) = (a \lor b) \wedge c
    \end{equation}
    holds. Hence distributivity may not hold.
\end{definition}

\begin{definition}[Complemented Lattice]
    A bounded lattice, with least and greatest elements denoted $0$ and $1$, respectively, is called complemented if every element $a$ has a complement $b$ such that $a\wedge b = 0$ and $a \lor b = 1$. Orthocomplementation is the operation that maps $a$ to $a^\perp$ such that $a\wedge a^\perp = 0$, $a \lor a^\perp = 1$ and $a^{\perp^\perp}=1$.
\end{definition}

\begin{definition}[Boolean Lattice]
    A Boolean lattice is a complemented distributive lattice.
\end{definition}

%% file: appendix_b.tex
\section{Proofs}
\label{sec:app_proofs}
\subsection{Proof of Lemma \ref{lemma:nuclear_symmetry}}
The nuclear norm on the left is equal to $\mathrm{Tr}\left(\left(\Yv^\top \Uv_1^\top \Uv_1 \Yv + \Vv^\top \Xv^\top \Uv_2^\top \Uv_2 \Xv \Vv\right)^{1/2}\right)$ which yields $\mathrm{Tr}\left(\left(\Yv^\top \Yv + \Vv^\top \Xv^\top \Xv\Vv\right)^{1/2}\right)$. Through further manipulation we arrive at $\mathrm{Tr}\left(\left((\Vv \Yv^\top \Yv \Vv^\top + \Xv^\top \Xv\right)^{1/2}\right)$. The nuclear norm on the right is $\mathrm{Tr}\left(\left(\Vv\Yv^\top \Yv\Vv^\top + \Xv^\top\Xv\right)^{1/2}\right)$ and thus equal to that on the left.

\subsection{Proof of Theorem \ref{theorem:nuclear_min_v}}
We prove the case where $d=c$. The other cases follow from this one. Invoking Lemma \ref{lemma:nuclear_symmetry}, let us remove the gauge freedom and rewrite the problem as 
\begin{align}
    \min_{\Vv_X \in \mathrm{St}_c(\mathbb{R}^n)}\;&\left\|\begin{bmatrix}\mathbf{\Sigma}_Y \Vv_Y^\top \\ \mathbf{\Sigma}_X \Vv_X^\top \end{bmatrix}\right\|_\ast.
    \label{eq:matrix_here}
\end{align}
Let $f(\Vv_X):=\left\|\begin{bmatrix}\mathbf{\Sigma}_Y \Vv_Y^\top \\ \mathbf{\Sigma}_X \Vv_X^\top \end{bmatrix}\right\|_\ast$. Since $\Vv_Y \in \mathrm{St}_c(\mathbb{R}^n)$, then $\Vv_Y \Vv_Y^\top$ is an orthogonal projection and thus
\begin{equation}
   f(\Vv_X) \geq \left\|\begin{bmatrix}\mathbf{\Sigma}_Y \Vv_Y^\top \\ \mathbf{\Sigma}_X \Vv_X^\top \end{bmatrix} \Vv_Y \Vv_Y^\top \right\|_\ast = \left\|\begin{bmatrix}\mathbf{\Sigma}_Y \\ \mathbf{\Sigma}_X \Vv_X^\top \Vv_Y\end{bmatrix}\Vv_Y^\top\right\|_\ast = \left\|\begin{bmatrix}\mathbf{\Sigma}_Y \\ \mathbf{\Sigma}_X \Vv_X^\top \Vv_Y\end{bmatrix}\right\|_\ast.
\end{equation}
Let $\Hv\in\mathbb{R}^{c\times c} = \Vv_X^\top\Vv_Y$ be the misalignment between $\Vv_X$ and $\Vv_Y$ and define the lower bound 
\begin{equation}
    g(\Hv) = \left\|\begin{bmatrix}\mathbf{\Sigma}_Y \\ \mathbf{\Sigma}_X \Hv\end{bmatrix}\right\|_\ast,
\end{equation}
which verifies $\max_{\|\Hv\|_2\leq 1} g(\Hv) \leq \min_{\Vv_X\in\mathrm{St}_c(\mathbb{R})^n} f(\Vv_X)$. Noting that $g(\Hv)$ can be written as 
\begin{equation}
    g(\Hv)=\mathrm{Tr}\left(\sqrt{\mathbf{\Sigma}_Y^2 + \Hv^\top\mathbf{\Sigma}_X^2 \Hv}\right),
\end{equation}
we have that it attains its maximum for $\Hv=\Iv_c$. We can verify that $f$ attains this lower bound for $\Vv_X = \Vv_Y$ \textit{i.e.}, 
\begin{equation}
    f(\Vv_Y) = \left\|\begin{bmatrix} \mathbf{\Sigma}_Y \Vv_Y^\top \\ \mathbf{\Sigma}_X \Vv_Y^\top \end{bmatrix}\right\|_\ast = \left\|\begin{bmatrix} \mathbf{\Sigma}_Y \\ \mathbf{\Sigma}_X \end{bmatrix}\right\|_\ast = \mathrm{Tr}\left(\sqrt{\mathbf{\Sigma}_Y^2 + \mathbf{\Sigma}_X^2}\right).
\end{equation}

\subsection{Proof of Theorem \ref{theorem:min_nuclear_norm_ell2_squared}}
Writing the SVD of $\Xv$ as $\Uv_X \mathbf{\Sigma}_X \Vv_X^\top$, where $\Uv_X\in O(d)$, $\mathbf{\Sigma}_X$ is diagonal and $\Vv_X\in \mathrm{St}_d(\mathbb{R}^n)$, the terms $\|\Xv\|_\ast$ and $\|\Xv\|_2$ only depend on $\mathbf{\Sigma}_X$. Let us rewrite the loss as
\begin{equation}
    \left\|\begin{bmatrix} \Yv \\ \Xv\end{bmatrix}\right\|_\ast - \alpha\|\mathbf{\Sigma}_X\|_\ast + \beta\|\mathbf{\Sigma}_X\|_2^2 = \left\|\begin{bmatrix} \Uv_Y \mathbf{\Sigma}_Y \Vv_Y^\top \\ \Uv_X\mathbf{\Sigma}_X \Vv_X^\top\end{bmatrix}\right\|_\ast - \alpha\mathrm{Tr}(\mathbf{\Sigma}_X) + \beta\|\mathbf{\Sigma}_X\|_2^2.
\end{equation}
Using Lemma \ref{lemma:nuclear_symmetry}, we have the equivalent problem
\begin{align}
    \min_{\mathbf{\Sigma}_X \succeq 0, \Vv_X\in \mathrm{St}_d(\mathbb{R}^n)}\left\{ \left\|\begin{bmatrix} \mathbf{\Sigma}_Y \Vv_Y^\top \\ \mathbf{\Sigma}_X \Vv_X^\top\end{bmatrix}\right\|_\ast - \alpha\mathrm{Tr}(\mathbf{\Sigma}_X) + \beta\|\mathbf{\Sigma}_X\|_2^2\right\}
\end{align}
or similarly,
\begin{align}
    \min_{\mathbf{\Sigma}_X \succeq  0}\;\left\{ \min_{\Vv_X\in \mathrm{St}_d(\mathbb{R}^n)}\left\{ \left\|\begin{bmatrix} \mathbf{\Sigma}_Y \Vv_Y^\top \\ \mathbf{\Sigma}_X \Vv_X^\top\end{bmatrix}\right\|_\ast \right\} - \alpha\mathrm{Tr}(\mathbf{\Sigma}_X) + \beta\|\mathbf{\Sigma}_X\|_2^2 \right\}.
\end{align}
Denote the singular values of $\Yv$ by $\mu_1 \geq \dots \geq \mu_c$ and those of $\Xv$ by $\sigma_1 \geq \dots \geq \sigma_d$. According to Theorem \ref{theorem:nuclear_min_v}, the inner minimization yields $\sum_{i=1}^c \sqrt{\mu_i^2 + \sigma_i^2} + \sum_{i=c+1}^{d}\sigma_i$ for $\Vv_X = \begin{bmatrix}\Vv_Y & \Vv \end{bmatrix}$. Thus,
\begin{align}
    \min_{\sigma_1\geq\dots\geq\sigma_d\geq 0}\;\left\{\sum_{i=1}^c \sqrt{\mu_i^2 + \sigma_i^2} + \sum_{i=c+1}^{d}\sigma_i - \alpha\sum_{i=1}^{d}\sigma_i  + \beta \sigma_1^2 \right\}.
    \label{eq:main_theorem_loss_in_sigma}
\end{align}
This problem is convex. Let us drop the constraint and consider the relaxed problem
\begin{align}
    &\min_{\sigma,t}\; \beta t^2 + \sum_{i=1}^c \left(\sqrt{\mu_i^2 + \sigma_i^2} - \alpha \sigma_i\right) + (1-\alpha)\sum_{i=c+1}^d \sigma_i \nonumber \\
    &\;\;\mathrm{s.t.}\;\; t - \sigma_i \geq 0,\quad i\in [c] \nonumber \\
    &\quad\quad\; \sigma_i \geq 0,\quad  i\in \{c+1,\dots,d\}
\end{align}
For dual variables $\lambda_i \geq 0$, $i\in[c]$ and $\nu_i \geq 0, i\in\{c+1,\dots,d\}$ we have the Lagrangian
\begin{equation}
    L(\sigma,t;\lambda,\nu) = \sum_{i=1}^c \left(\sqrt{\mu_i^2 + \sigma_i^2}  +(\lambda_i-\alpha)\sigma_i\right) + \beta t^2 - \sum_{i=1}^c\lambda_i t + \sum_{i=c+1}^d\left( (1-\alpha)\sigma_i-\nu_i\sigma_i\right).
\end{equation}
The KKT conditions sufficient for the optimality of the relaxation are thus
\begin{equation}
    \begin{cases}
        t-\sigma_i\geq0,\;i\in[c] \quad&\text{(a - primal\;feasibility)} \\
        \sigma_i\geq0,\;i\in\{c+1,\dots,d\} \quad&\text{(b - primal\;feasibility)}\\
        \lambda_i \geq 0,\; i\in[c]\quad&\text{(c - dual \;feasibility)}\\
        \nu_i \geq 0,\; i\in\{c+1,\dots,d\}\quad&\text{(d - dual \;feasibility)} \\
        \frac{\sigma_i}{\sqrt{\sigma_i^2 + \mu_i^2}} + \lambda_i-\alpha = 0,\;i\in[c] \quad&\text{(e - stationarity)} \\
        \nu_i=1-\alpha,\;i\in\{c+1,\dots,d\} \quad&\text{(f - stationarity)}\\ 
        t =\frac{1}{2\beta }\sum_{i=1}^c\lambda_i \quad&\text{(g - stationarity)} \\ 
        \sum_{i=1}^c \lambda_i(t-\sigma_i)=0\quad&\text{(h - complementary\;slackness)} \\
        \sum_{i=c+1}^d \nu_i\sigma_i=0\quad&\text{(i - complementary\;slackness)}.
    \end{cases}
    \label{eq:primal_kkt_main_theorem}
\end{equation}
Let $\sigma_i^\ast=t^\ast,\;i\in[c]$ and $\sigma_i^\ast=0,\;i\in\{c+1,\dots,d\}$. Primal feasibility (a), (b) and complementary slackness (h), (i) hold. Since $\alpha < 1$, we have $\nu_i^\ast=1-\alpha \geq 0,\;i\in\{c+1,\dots,d
\}$. Thus, dual feasibility (d) and stationarity (f) hold. Plugging $\sigma_i^\ast=t^\ast,\;i\in[c]$ in the stationarity conditions (e) and (g) yields
\begin{align}
    \begin{cases}\frac{t^\ast}{\sqrt{{t^\ast}^2+\mu_i^2}}+\lambda_i^\ast -\alpha = 0,\; i\in[c] \\ t^\ast=\frac{1}{2\beta}\sum_{i=1}^c \lambda_i^\ast.\end{cases}
\end{align}
Solving for $t^\ast$, we obtain $\alpha c - 2\beta t^\ast = \sum_{i=1}^c\frac{t^\ast}{\sqrt{\mu_i^2+{t^\ast}^2}}$ and thus $0 < t^\ast\leq\frac{\alpha c}{2\beta}$. It suffices to verify dual feasibility (c). Plugging this upper bound in the stationarity condition (e) we obtain
\begin{align}
    \lambda_i^\ast=\alpha-\frac{t^\ast}{\sqrt{\mu_i^2+{t^\ast}^2}}\geq \alpha - \frac{\alpha c}{2\beta\sqrt{\mu_c^2 + (\alpha c/2\beta)^2}}\geq 0.
\end{align}
We can thus find $\alpha$ that ensures dual feasibility (c), which yields $\alpha\geq\sqrt{\max\left\{0,1-\frac{4\beta^2\mu_c^2}{c^2}\right\}}$. All KKT conditions for the relaxation hold. To show that this solution solves the original problem, note that $\sigma_1^\ast = \dots = \sigma_c^\ast = t^\ast > 0 $ and $\sigma_{c+1}^\ast=\dots\sigma_{d}^\ast = 0$ imply $\sigma_1^\ast\geq\dots\geq\sigma_d^\ast\geq 0$.

Concluding the proof, the solution set is 
\begin{equation}
    \mathcal{X} = \left\{ \Uv (t^\ast \Iv_c) \begin{bmatrix} \Vv_Y^\top \\ \Vv^\top\end{bmatrix} \Bigg|\; \Uv\in \mathrm{St}_c(\mathbb{R}^d),\; \Vv \in \mathcal{N}(\Vv_Y)\right\},
\end{equation}
where $t^\ast$ is the solution to $\alpha c - 2\beta t = \sum_{i=1}^c\frac{t}{\sqrt{\mu_i^2+{t}^2}}$.

\subsection{Proof of Lemma \ref{lemma:min_nuclear_norm_ell2}}
\begin{proof}
     Writing the SVD of $\Xv$ as $\Uv_X \mathbf{\Sigma}_X \Vv_X^\top$, where $\Uv_X\in O(d)$, $\mathbf{\Sigma}_X$ is diagonal and $\Vv_X\in \mathrm{St}_d(\mathbb{R}^n)$, the terms $\|\Xv\|_\ast$ and $\|\Xv\|_2$ only depend on $\mathbf{\Sigma}_X$. Let us rewrite the loss as
    \begin{equation}
        \left\|\begin{bmatrix} \Yv \\ \Xv\end{bmatrix}\right\|_\ast - \|\mathbf{\Sigma}_X\|_\ast + \beta\|\mathbf{\Sigma}_X\|_2 = \left\|\begin{bmatrix} \Uv_Y \mathbf{\Sigma}_Y \Vv_Y^\top \\ \Uv_X\mathbf{\Sigma}_X \Vv_X^\top\end{bmatrix}\right\|_\ast - \mathrm{Tr}(\mathbf{\Sigma}_X) + \beta\|\mathbf{\Sigma}_X\|_2.
    \end{equation}
    Using Lemma \ref{lemma:nuclear_symmetry}, we have
    \begin{align}
        \min_{\mathbf{\Sigma}_X \succeq 0, \Vv_X\in \mathrm{St}_d(\mathbb{R}^n)}\left\{ \left\|\begin{bmatrix} \mathbf{\Sigma}_Y \Vv_Y^\top \\ \mathbf{\Sigma}_X \Vv_X^\top\end{bmatrix}\right\|_\ast - \mathrm{Tr}(\mathbf{\Sigma}_X) + \beta\|\mathbf{\Sigma}_X\|_2\right\}
    \end{align}
    or equivalently,
     \begin{align}
        \min_{\mathbf{\Sigma}_X \succeq  0}\;\left\{ \min_{\Vv_X\in \mathrm{St}_d(\mathbb{R}^n)} \left\|\begin{bmatrix} \mathbf{\Sigma}_Y \Vv_Y^\top \\ \mathbf{\Sigma}_X \Vv_X^\top\end{bmatrix}\right\|_\ast - \mathrm{Tr}(\mathbf{\Sigma}_X) + \beta\|\mathbf{\Sigma}_X\|_2 \right\}.
    \end{align}
     Denoting the singular values of $\Yv$ by $\mu_1 \geq \dots \geq \mu_c$ and those of $\Xv$ by $\sigma_1 \geq \dots \geq \sigma_d$ then, according to Theorem \ref{theorem:nuclear_min_v}, the inner minimization yields $\sum_{i=1}^c \sqrt{\mu_i^2 + \sigma_i^2} + \sum_{i=c+1}^{d}\sigma_i$. Thus,
    \begin{align}
        &\min_{\sigma_1\geq\dots\geq\sigma_d\geq 0}\;\left\{\sum_{i=1}^c \sqrt{\mu_i^2 + \sigma_i^2} + \sum_{i=c+1}^{d}\sigma_i - \sum_{i=1}^{d}\sigma_i  + \beta \sigma_1 \right\}
        \nonumber \\
        =&\min_{\sigma_1\geq\dots\geq\sigma_c\geq 0}\;\left\{\sum_{i=1}^c\left(\sqrt{\mu_i^2 + \sigma_i^2} - \sigma_i\right)  + \beta \sigma_1 \right\}.
    \end{align}
    This problem is convex and independent of the values of $\sigma_i$ for $i > c$. Let us relax the constraint $\sigma_1\geq\dots\geq\sigma_d\geq 0$ and consider the problem
    \begin{align}
        &\min_{\sigma,t}\; \beta t + \sum_{i=1}^c\left( \sqrt{\mu_i^2 + \sigma_i^2} - \sigma_i\right) \nonumber \\
        &\;\;\mathrm{s.t.}\;\; t - \sigma_i \geq 0,\quad i\in [c].
    \end{align}
    For the dual variables $\lambda_i \geq 0$, $i\in [c]$, we have the Lagrangian
    \begin{equation}
        L(\sigma,t;\lambda) = \beta t + \sum_{i=1}^c\left( \sqrt{\mu_i^2 + \sigma_i^2} - \sigma_i - \lambda_i(t-\sigma_i)\right),
    \end{equation}
    with KKT conditions
    \begin{equation}
        \begin{cases}
            t-\sigma_i \geq 0,\; i \in [c] \quad&\text{(a - primal feasibility)}\\
            \lambda_i \geq 0,\; i \in [c] \quad&\text{(b - dual feasibility)}\\
            \frac{\sigma_i}{\sqrt{\mu_i^2+\sigma_i^2}}+\lambda_i-1 = 0,\; i \in [c] \quad&\text{(c - stationarity)}\\
            \sum_{i=1}^c \lambda_i = \beta \quad&\text{(d - stationarity)}\\
            \sum_{i=1}^c \lambda_i(t-\sigma_i)=0\quad&\text{(e - complementary slackness)}. \\
        \end{cases}
    \end{equation}
    Let $\sigma_i^\ast = t^\ast$, $i\in[c]$. Then, primal feasibility (a) and complementary slackness (e) hold. In order to have stationarity (c), we set
    \begin{align}
        \lambda_i^\ast = 1 - \frac{t^\ast}{\sqrt{\mu_i^2+t^{\ast^2}}},\; i\in[c]
    \end{align}
    and thus $\lambda_i^\ast \geq 0$, $i\in[c]$ and dual feasibility holds. Solving (d) yields
    \begin{align}
        \sum_{i=1}^c\frac{t^\ast}{\sqrt{\mu_i^2+t^{\ast^2}}} = c - \beta,
    \end{align}
    from where we find that $t^\ast > 0$. Thus, the solution of the relaxed problem $\sigma_i^\ast = t^\ast$, for $i \in [c]$ is also feasible for the original problem, provided we pick $\sigma_{c+1},\dots,\sigma_{d}$ such that $t^\ast \geq \sigma_{c+1}\geq \dots\geq\sigma_{d}\geq 0$. Concluding the proof, the solution set is 
    \begin{equation}
        \mathcal{X} = \left\{ \Uv \begin{bmatrix}t^\ast \Iv_c & \mathbf{0} \\ \mathbf{0} & \mathbf{\Sigma}\end{bmatrix} \begin{bmatrix} \Vv_Y^\top \\ \Vv^\top\end{bmatrix} \Bigg|\; \Uv\in O(d),\; t^\ast \Iv \succeq \mathbf{\Sigma},\; \Vv \in \mathcal{N}(\Vv_Y)\right\},
    \end{equation}
    where $t^\ast$ is the solution to $\sum_{i=1}^c\frac{t}{\sqrt{\mu_i^2+t^{^2}}} = c - \beta$.
\end{proof}

\subsection{Proof of Lemma \ref{lemma:min_nuclear_norm_no_ell2}}
If we write the SVD of $\Xv$ as $\Uv_X \mathbf{\Sigma}_X \Vv_X^\top$, where $\Uv_X\in O(d)$, $\mathbf{\Sigma}_X$ is diagonal and $\Vv_X\in \mathrm{St}_d(\mathbb{R}^n)$ the term $\|\Xv\|_\ast$ is simply the trace of $\mathbf{\Sigma}_X$. Let us rewrite the loss as
\begin{equation}
    \left\|\begin{bmatrix} \Yv \\ \Xv\end{bmatrix}\right\|_\ast - \alpha \|\mathbf{\Sigma}_X\|_\ast = \left\|\begin{bmatrix} \Uv_Y \mathbf{\Sigma}_Y \Vv_Y^\top \\ \Uv_X\mathbf{\Sigma}_X \Vv_X^\top\end{bmatrix}\right\|_\ast - \alpha\mathrm{Tr}(\mathbf{\Sigma}_X).
\end{equation}
Using Lemma \ref{lemma:nuclear_symmetry}, we have
\begin{align}
    \min_{\mathbf{\Sigma}_X \succeq 0, \Vv_X\in \mathrm{St}_d(\mathbb{R}^n)}\; \left\|\begin{bmatrix} \mathbf{\Sigma}_Y \Vv_Y^\top \\ \mathbf{\Sigma}_X \Vv_X^\top\end{bmatrix}\right\|_\ast - \alpha\mathrm{Tr}(\mathbf{\Sigma}_X)
\end{align}
or equivalently,
\begin{align}
    \min_{\mathbf{\Sigma}_X \succeq 0}\;\left\{ \min_{\Vv_X\in \mathrm{St}_d(\mathbb{R}^n)} \left\|\begin{bmatrix} \mathbf{\Sigma}_Y \Vv_Y^\top \\ \mathbf{\Sigma}_X \Vv_X^\top\end{bmatrix}\right\|_\ast - \alpha\mathrm{Tr}(\mathbf{\Sigma}_X)\right\}.
\end{align}
Denote the singular values of $\Yv$ by $\mu_1 \geq \dots \geq \mu_c$ and those of $\Xv$ by $\sigma_1 \geq \dots \geq \sigma_d$. According to Theorem \ref{theorem:nuclear_min_v}, the inner minimization yields $\sum_{i=1}^c \sqrt{\mu_i^2 + \sigma_i^2} + \sum_{i=c+1}^{d}\sigma_i$. Thus,
\begin{align}
    &\min_{\sigma_1\geq\dots\geq\sigma_d\geq 0}\;\left\{\sum_{i=1}^c \sqrt{\mu_i^2 + \sigma_i^2} + \sum_{i=c+1}^{d}\sigma_i - \alpha \sum_{i=1}^{d}\sigma_i  \right\}.
\end{align}
This problem is convex. Let us relax the constraints $\sigma_1\geq\dots\geq\sigma_d\geq 0$, replacing them by $\sigma_i\geq 0, i\in [d]$. In this relaxation, the derivative w.r.t. $\sigma_k$ yields $\frac{\sigma_k}{\sqrt{\mu_k+\sigma_k^2}}-\alpha$ for $k\in[c]$. The first-order stationarity condition puts thus the optimum at $\sigma_i^\ast = \frac{\alpha}{\sqrt{1-\alpha^2}}\mu_i$ for $i\in[c]$. For $k > c$, the derivative w.r.t. $\sigma_k$ yields $1-\alpha > 0$. Given the relaxed constraint $\sigma_i \geq 0, i\in[d]$, the minimum is therefore attained for $\sigma_i^\ast = 0$, $i\in\{c+1,\dots,d\}$. Note that 
\begin{align}
\sigma_i^\ast =
    \begin{cases}
        \frac{\alpha}{\sqrt{1-\alpha^2}}\mu_i,\quad & i\in [c] \\
        0,\quad & i \in \{c+1,\dots,d\}
    \end{cases}
\end{align}
is feasible for the original problem and thus optimal as well. The solution set is thus given by
\begin{equation}
    \mathcal{X}=\left\{\Uv \begin{bmatrix}\frac{\alpha}{\sqrt{1-\alpha^2}}\mathbf{\Sigma}_Y & \mathbf{0} \\ \mathbf{0} & \mathbf{0} \end{bmatrix} \begin{bmatrix}\Vv_Y^\top \\ \Vv^\top\end{bmatrix} \;\Bigg|\; \Uv \in O(d),\; \Vv \in \mathcal{N}(\Vv_Y)\right\}.
\end{equation}

\subsection{Proof of Lemma \ref{lemma:orthogonal_vs}}
We start by showing that $\forall i,j \in [n]\; \yv_i = \yv_j \Leftrightarrow \vv_i = \vv_j$. Each column of $\mathbf{\Sigma} \Vv^\top$ can be written as an orthogonal transformation of each column of $\Yv$ by $\Uv^\top\in O(c)$. Therefore, $\yv_i = \yv_j \Leftrightarrow \Uv^\top \yv_i = \Uv^\top \yv_j \Leftrightarrow \mathbf{\Sigma} \vv_i = \mathbf{\Sigma} \vv_j$. Since $\mathbf{\Sigma}$ is square and full-rank, $\mathbf{\Sigma} \vv_i = \mathbf{\Sigma} \vv_j \Leftrightarrow \vv_i = \vv_j$. Together with the definition of $\mathcal{I}$, this implies that $\forall i,j\in\mathcal{I}$ $\vv_i \neq \vv_j$. For the second part, since $\Vv \in \mathrm{St}_c(\mathbb{R}^n)$ we have $\Vv^\top \Vv = \Iv_c$. Thus, $\forall \xv\in\mathbb{R}^c$, $\Vv^\top \Vv \xv  = \xv$. For $i \in \mathcal{I}$, define the sets $\mathcal{J}_i = \{j \in [n] : \yv_j = \yv_i\}$ and note that $\bigcup_{i \in \mathcal{I}} \mathcal{J}_i = [n]$. We can write $\Vv^\top \Vv \xv$ as 
\begin{align}
    \sum_{i=1}^n \vv_i \langle \vv_i, \xv\rangle  = \sum_{i\in\mathcal{I}}\sum_{j\in\mathcal{J}_i} \vv_j \langle \vv_j, \xv\rangle = 
    \sum_{i\in\mathcal{I}} |\mathcal{J}_i| \vv_i \langle \vv_i, \xv\rangle = 
    \sum_{i\in\mathcal{I}} \sqrt{|\mathcal{J}_i|} \vv_i \langle \sqrt{|\mathcal{J}_i|}\vv_i, \xv\rangle,
    \label{eq:equal_cols_proof}
    \end{align}
where we used the fact that, by definition, for any two indices $l,k \in \mathcal{J}_i$ we must have $\vv_l=\vv_k=\vv_i$. We can write (\ref{eq:equal_cols_proof}) in matrix notation by defining the matrix $\bar{\Vv}$ with columns $\{\sqrt{|\mathcal{J}_i|} \vv_i\}_{i\in\mathcal{I}}$.  The identity $\forall \xv\in\mathbb{R}^c\;\Vv^\top \Vv \xv = \xv$ becomes $\bar{\Vv}^\top \bar{\Vv} \xv = \xv$. From the hypothesis that $|\mathcal{I}| = \mathrm{rank}\; \Yv = c$, $\bar{\Vv}$ is square and full-rank. Thus, $\forall \xv\in\mathbb{R}^c\;\bar{\Vv}^\top \bar{\Vv} \xv = \xv$ implies that $\bar{\Vv}^\top \bar{\Vv} = \Iv_c$ and we have $\bar{\Vv}\in O(c)$. Therefore, the columns of $\bar{\Vv}$,  $\{\sqrt{|\mathcal{J}_i|} \vv_i\}_{i\in\mathcal{I}}$, form an orthonormal basis for $\mathbb{R}^c$.

\subsection{Proof of Lemma \ref{lemma:projections_are_propositions}}
Since the minterms $\{\yv_i\}_{i\in\mathcal{I}}$ correspond to disjoint events, the posterior is given by the probability of the disjunction of all $\yv_i$ that imply $\qv$. Thus,
\begin{align}
    P(\qv|\xv) &= P\left(\bigvee_{i\in\mathcal{I}\;\mathrm{st}\; \yv_i \implies \qv} \yv_i\bigg| \xv\right) \stackrel{(a)}{=} \sum_{i\in\mathcal{I}\;\mathrm{st}\;  \yv_i \implies \qv} P(\yv_i | \xv) \stackrel{(b)}{=}\sum_{i\in\mathcal{I}\;\mathrm{st}\;  \yv_i \implies \qv} \left\langle \xv^\top \ev_i, \ev_i^\top \xv\right\rangle \nonumber \\ &=\left\langle \xv, \left(\sum_{i\in\mathcal{I}\;\mathrm{st}\;  \yv_i \implies \qv} \ev_i \ev_i^\top \right) \xv\right\rangle \stackrel{(c)}{=} \langle \xv, \Pv_q \xv\rangle,
    \label{eq:posterior_probability}
\end{align}
(a) $\{\yv_i\}_{i\in\mathcal{I}}$ correspond to disjoint events; (b) definition of $P(\yv_i|\xv)$; (c) $\Pv_q$ is a projection operator onto the subspace spanned by $\{\ev_i\}_{i\in\mathcal{I}\;\mathrm{st}\; \yv_i \implies \qv}$.

\subsection{Additional theoretical results}
If we fix the singular values of $\Xv$, the nuclear norm of $\Zv$ can easily be upper bounded. This stems from the triangle inequality
\begin{equation}
    \left\|\begin{bmatrix}
        \Yv \\
        \Xv
    \end{bmatrix}\right\|_\ast \leq \left\| \begin{bmatrix} \Yv \\ \mathbf{0}_{d\times n} \end{bmatrix}\right\|_\ast + \left\|\begin{bmatrix} \mathbf{0}_{c\times n} \\ \Xv \end{bmatrix}\right\|_\ast
\end{equation}
being tight if we pick $\Xv$ such that $\Yv\Xv^\top = \Yv^\top \Xv = 0$ (the nuclear norm is additive if the matrices have orthogonal row and column spaces \cite{recht2010guaranteed}). Hence if $\mathbf{\Sigma}_Y$ denotes the singular values of $\Yv$
\begin{equation}
    \max_{\Vv_X\in\mathrm{St}_d(\mathbb{R}^d)} \left\|\begin{bmatrix}
        \Yv \\
        \Uv_X \mathbf{\Sigma}_X \Vv_X^\top
    \end{bmatrix}\right\|_\ast = \mathrm{Tr}\left(\mathbf{\Sigma}_Y + \mathbf{\Sigma}_X\right)
\end{equation}

\begin{lemma}
    Let $\Xv\in\mathbb{R}^{d\times n}$, with $n>d$, be a matrix with unit $\ell_2$-norm columns. Then $\max_{\Xv} \; \|\Xv\|_\ast = d\sqrt{n/d}$ and $\mathrm{argmax}_{\Xv} \; \|\Xv\|_\ast=\sqrt{n/d} \;\Uv\Vv^\top$, for some $\Uv\in O(d)$ and $\Vv\in\mathrm{St}_d(\mathbb{R}^n)$.
\end{lemma}
\begin{proof}
    Let the SVD be $\Xv=\Uv\mathbf{\Sigma}\Vv^\top$. Then, $\Xv^\top \Xv= \Vv\mathbf{\Sigma}^2 \Vv^\top$ and $\mathrm{Tr}(\Xv^\top \Xv)=\sum_{i=1}^d\sigma_i^2$. From the normalization, it follows that $\mathrm{Tr}(\Xv^\top \Xv)=\sum_{i=1}^n x_i^\top x_i=n$. The maximization of $\|\Xv\|_\ast$ corresponds thus to the optimization problem
    \begin{align}
        &\max_{\{\sigma_i\}_{i=1}^d}\;\sum_{i=1}^d\sigma_i \nonumber \\ 
        &\;\;\mathrm{s.t.}\sum_{i=1}^d \sigma_i^2=n.
    \end{align}
    The gradient of the cost function is $\mathbf{1}\in\mathbb{R}^d$ and his orthogonal to the feasible set at $\sigma_i=\sqrt{\frac{n}{d}},\; i=1,\dots,d$. Thus, $\max_{\Xv}\;\|\Xv\|_\ast=d\sqrt{n/d}$.
\end{proof}

\begin{lemma}
    \label{lemma:spectral_and_svd}
    For any $\Xv\in\mathbb{R}^{n\times n}$ with singular values $\sigma_1 \geq \dots \geq \sigma_n > 0$, 
    \begin{equation}
        \mathrm{spectrum}\left(\begin{bmatrix}\mathbf{0}_n & \Xv^\top \\ \Xv & \mathbf{0}_n\end{bmatrix}\right) = \bigcup_{i=1}^n \{-\sigma_i, \sigma_i\}
    \end{equation}
    \begin{proof}
        Denoting the SVD of $\Xv$ by $\Uv\mathbf{\Sigma}\Vv^\top$, it suffices to check that 
        \begin{equation}
            \begin{bmatrix}\mathbf{0}_n & \Xv^\top \\ \Xv & \mathbf{0}_n\end{bmatrix} = \begin{bmatrix}\Vv/\sqrt{2} & -\Vv/\sqrt{2} \\ \Uv/\sqrt{2} & \Uv/\sqrt{2}\end{bmatrix}\begin{bmatrix}\mathbf{\Sigma} & \mathbf{0}_n \\ \mathbf{0}_n & -\mathbf{\Sigma}\end{bmatrix}\begin{bmatrix}\Vv^\top/\sqrt{2} & \Uv^\top/\sqrt{2} \\ -\Vv^\top/\sqrt{2} & \Uv^\top/\sqrt{2}\end{bmatrix}
        \end{equation}
        is the spectral decomposition of the symmetric matrix with $\Xv$ and $\Xv^\top$ in the off-diagonals.
    \end{proof}
\end{lemma}

\begin{lemma} Let $\Vv_Y, \Vv_X\in\mathrm{St}_c(\mathbb{R}^n)$, for $n>c$, and denote by $\{\sigma_i\}_{i=1}^{2c}$ the singular values of $\Zv \in \mathbb{R}^{2c \times n}$, defined as
\begin{equation}
    \Zv:= \begin{bmatrix}
        \Vv_Y^\top \\
        \Vv_X^\top
    \end{bmatrix}.
\end{equation}
Then, $\{\sigma_i\}_{i=1}^{2c} = \bigcup_{i=1}^{c}\left\{\sqrt{1 + \sigma_i(\Vv_Y^\top \Vv_X)}, \sqrt{1 - \sigma_i(\Vv_Y^\top \Vv_X)}\right\}$.
\begin{proof} Letting the full SVD of $\Zv$ be $\Uv\mathbf{\Sigma}\Vv^\top$, we have $\Zv\Zv^\top = \Uv\mathbf{\Sigma}^2 \Uv^\top$. Thus, its singular values are given by the square roots of the eigenvalues of $\Zv\Zv^\top$. From $\Vv_Y^\top\Vv_Y = \Vv_X^\top\Vv_X = \Iv_c$ we have
\begin{align}
    \sqrt{\lambda_i\left(\Zv\Zv^\top\right)} &= \sqrt{\lambda_i\left(\begin{bmatrix}
        \Iv_c & \Vv_Y^\top \Vv_X \\
        \Vv_X^\top \Vv_Y & \Iv_c
    \end{bmatrix}\right)} = \sqrt{1 + \lambda_i\left(\begin{bmatrix}
        \mathbf{0}_c & \Vv_Y^\top \Vv_X \\
        \Vv_X^\top \Vv_Y & \mathbf{0}_c
    \end{bmatrix}\right)}.
\end{align}
    Using Lemma \ref{lemma:spectral_and_svd}, this yields $\bigcup_{i=1}^{c}\left\{\sqrt{1 + \sigma_i(\Vv_Y^\top \Vv_X)}, \sqrt{1 - \sigma_i(\Vv_Y^\top\Vv_X)}\right\}$.
\end{proof}
\label{lemma:singular_values_stiefel}
\end{lemma}

\begin{lemma}
    For $\Vv_Y \in\mathrm{St}_c(\mathbb{R}^n)$, with $n>c$,
    \begin{align}
        \min_{\Vv_X \in \mathrm{St}_c(\mathbb{R}^n)} \left\|\begin{bmatrix} \Vv_Y^\top \\ \Vv_X^\top\end{bmatrix}\right\|_\ast  = \left\|\begin{bmatrix} \Vv_Y^\top \\ \Vv_Y^\top\end{bmatrix}\right\|_\ast = \sqrt{2}c
    \end{align}
\end{lemma}
\begin{proof}
    From Lemma \ref{lemma:singular_values_stiefel} we can write:
    \begin{equation}
        \left\|\begin{bmatrix} \Vv_Y^\top \\ \Vv_X^\top\end{bmatrix}\right\|_\ast = \sum_{i=1}^c \sqrt{1 + \sigma_i(\Vv_Y^\top \Vv_X)} + \sqrt{1 - \sigma_i(\Vv_Y^\top \Vv_X)}
    \end{equation}
    Note that $0 \leq \sigma_i(\Vv_Y^\top \Vv_X) \leq 1$ for $i\in [c]$, and
    \begin{equation}
        \sqrt{1 + \sigma_i} + \sqrt{1 - \sigma_i} \geq \sqrt{2},\quad \sigma_i \in [0,1].
    \end{equation}
    Therefore we have the lower bound
    \begin{equation}
        \left\|\begin{bmatrix} \Vv_Y^\top \\ \Vv_X^\top\end{bmatrix}\right\|_\ast \geq \sum_{i=1}^c \sqrt{2} = \sqrt{2}c.
    \end{equation}
    This bound is tight for $\Vv_X = \Vv_Y$, since $\sigma_i(\Vv_Y^\top \Vv_Y) = \sigma_i(\Iv_c) = 1, i\in [c]$.
\end{proof}

%% file: appendix_c.tex
\section{Experimental Details}
\label{sec:app_experiments}

\definecolor{bg}{gray}{0.95}
\DeclareTCBListing{mintedbox}{O{}m!O{}}{%
  breakable=true,
  listing engine=minted,
  listing only,
  minted language=#2,
  minted style=default,
  minted options={%
    linenos,
    gobble=0,
    breaklines=true,
    breakafter=,,
    fontsize=\small,
    numbersep=8pt,
    #1},
  boxsep=0pt,
  left skip=0pt,
  right skip=0pt,
  left=25pt,
  right=0pt,
  top=3pt,
  bottom=3pt,
  arc=5pt,
  leftrule=0pt,
  rightrule=0pt,
  bottomrule=2pt,
  toprule=2pt,
  colback=bg,
  colframe=blue!10,
  enhanced,
  overlay={%
    \begin{tcbclipinterior}
    \fill[blue!10!white] (frame.south west) rectangle ([xshift=20pt]frame.north west);
    \end{tcbclipinterior}},
  #3}

\subsection{PyTorch Implementation}
\label{sec:app_implementation}
Below we present an implementation of the loss function (\ref{eq:the_loss}) in PyTorch, called \texttt{NuclearLoss}, as well as the function for computing the minterm directions entitled \texttt{compute\_minterms\_vec}.
\begin{mintedbox}{python}
class NuclearLoss(nn.Module):
    def __init__(self, alpha : float, beta : float):
        super(Nuclear, self).__init__()
        self.a, self.b = alpha, beta

    def forward(self, x : torch.Tensor, y : torch.Tensor):
        z = torch.cat((y.permute(1,0),
                       x.permute(1,0)), dim=0)
        
        x_s = torch.linalg.svdvals(x)
        z_s = torch.linalg.svdvals(z)
        loss = z_s.sum() - self.a * x_s.sum() + self.b * x_s.max()**2
        return loss
\end{mintedbox}

\begin{mintedbox}{python}
def compute_minterms_vec(x : torch.Tensor,
                         y : torch.Tensor,
                         minterms : torch.Tensor):
    minterms_vec = []
    for minterm in minterms:
        mask = (y == minterm).all(dim=-1)
        u, _, _ = torch.linalg.svd(x[mask, :].T)
        minterms_vec.append(u[:,:1].T)
    return torch.cat(minterms_vec, dim=0)
\end{mintedbox}

The PyTorch code for the image augmentations used with CIFAR-10 and CIFAR-100 is provided below. 
\begin{mintedbox}{python}
class CIFARTransform():
    def __init__(self, split):
        mu  = (0.4914, 0.4822, 0.4465)
        std = (0.2023, 0.1994, 0.2010)

        if split.lower() == "train":
            self.t = Compose([RandomCrop(32, padding=4),
                              RandomHorizontalFlip(),
                              ToTensor(),
                              Normalize(mu,std)])
        else:
            self.t = Compose([ToTensor(),
                              Normalize(mu,std)])
\end{mintedbox}

The PyTorch code for the image augmentations used with CelebA is provided below.
\begin{mintedbox}{python}
class CELEBATransform():
    def __init__(self, split):
        if split.lower() == "train":
            self.t = Compose([Resize((40, 40)),
                              RandomHorizontalFlip(),
                              ToTensor()])
        else:
            self.t = Compose([Resize((40, 40)),
                              ToTensor()])
\end{mintedbox}

\subsection{Synthetic Experiments}
\label{sec:app_synthetic_experiments}
We present in Fig. \ref{fig:all_synthetic_experiments} additional comparisons between our approach, OLE, and MMCR for three binary label matrices $\Yv$. Figs. \ref{fig:Y1}, \ref{fig:Y2} and \ref{fig:Y0} show the Gram matrices of the labels $\Yv^\top \Yv$ while the remaining figures display the Gram matrices of the optimized representations. While MMCR and our approach show the same orthogonalizing behavior in the case of disjoint labels, results are drastically different when co-occurring labels are present. Our loss guarantees orthogonal minterms, thus generalizing the disjoint label setting.

\begin{figure}
\centering
\begin{subfigure}{0.24\linewidth}
    \includegraphics[width=\linewidth]{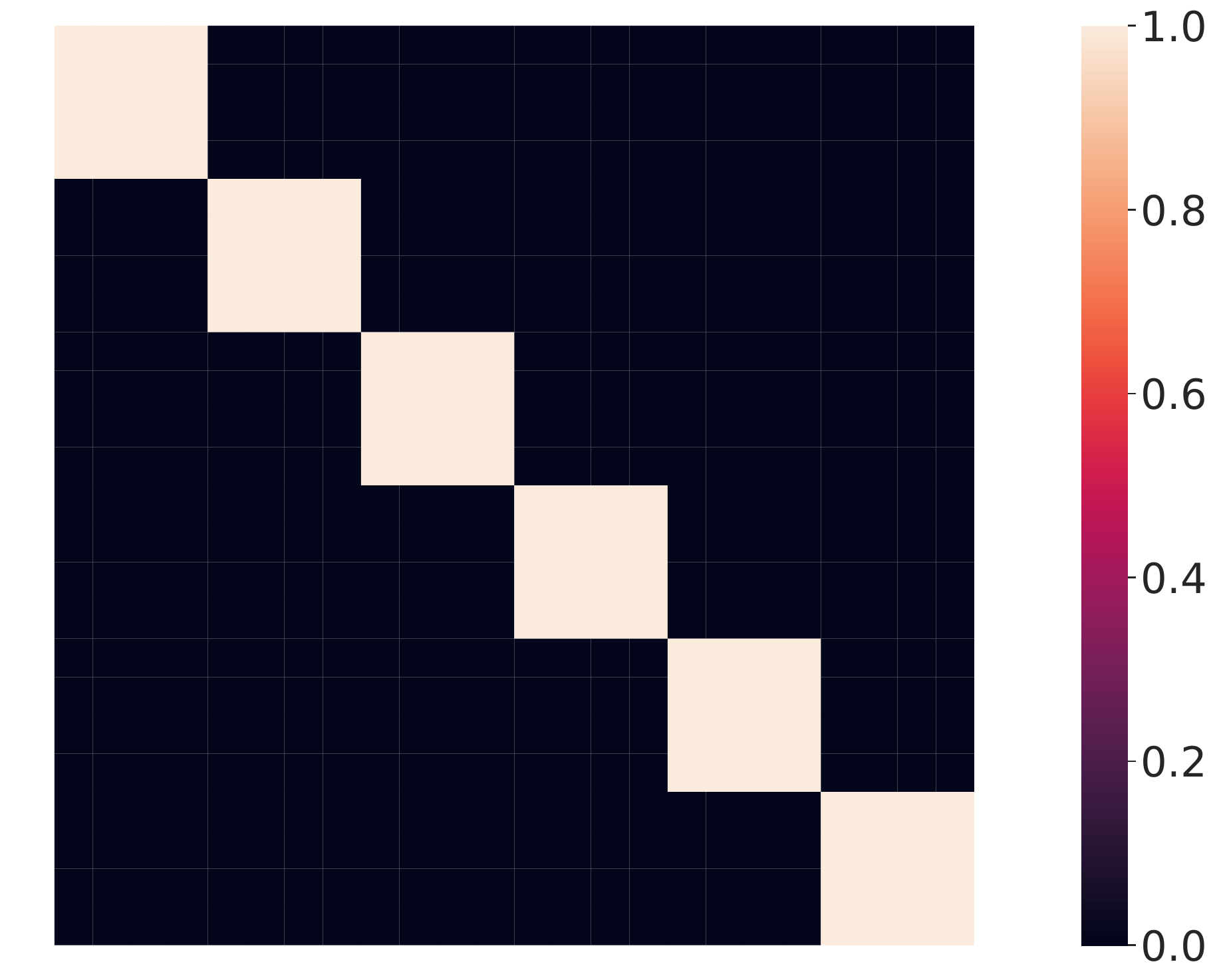}
    \caption{$\Yv^\top \Yv$}
    \label{fig:Y1}
\end{subfigure}
\begin{subfigure}{0.24\linewidth}
    \includegraphics[width=\linewidth]{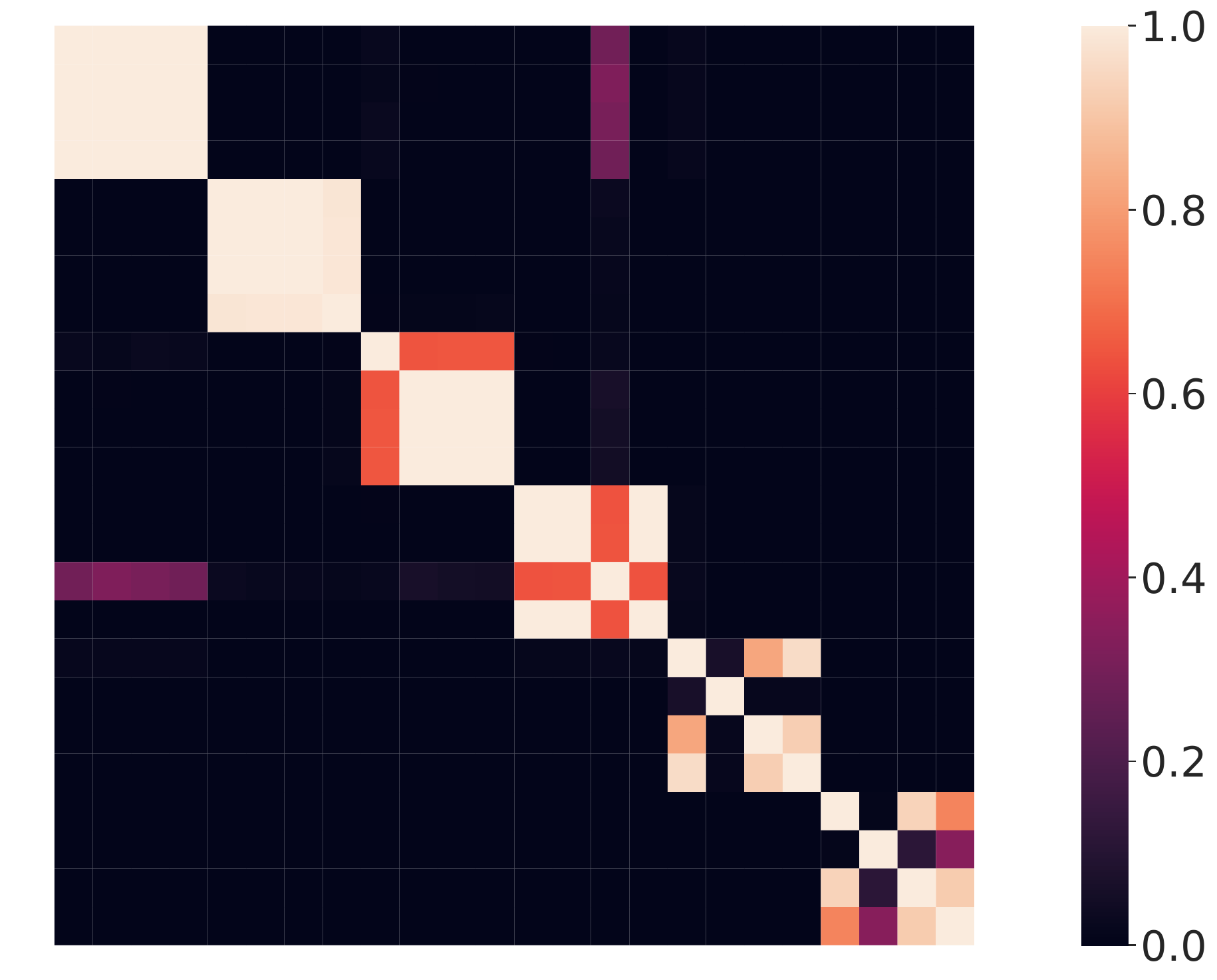}
    \caption{OLE}
\end{subfigure}
\begin{subfigure}{0.24\linewidth}
    \includegraphics[width=\linewidth]{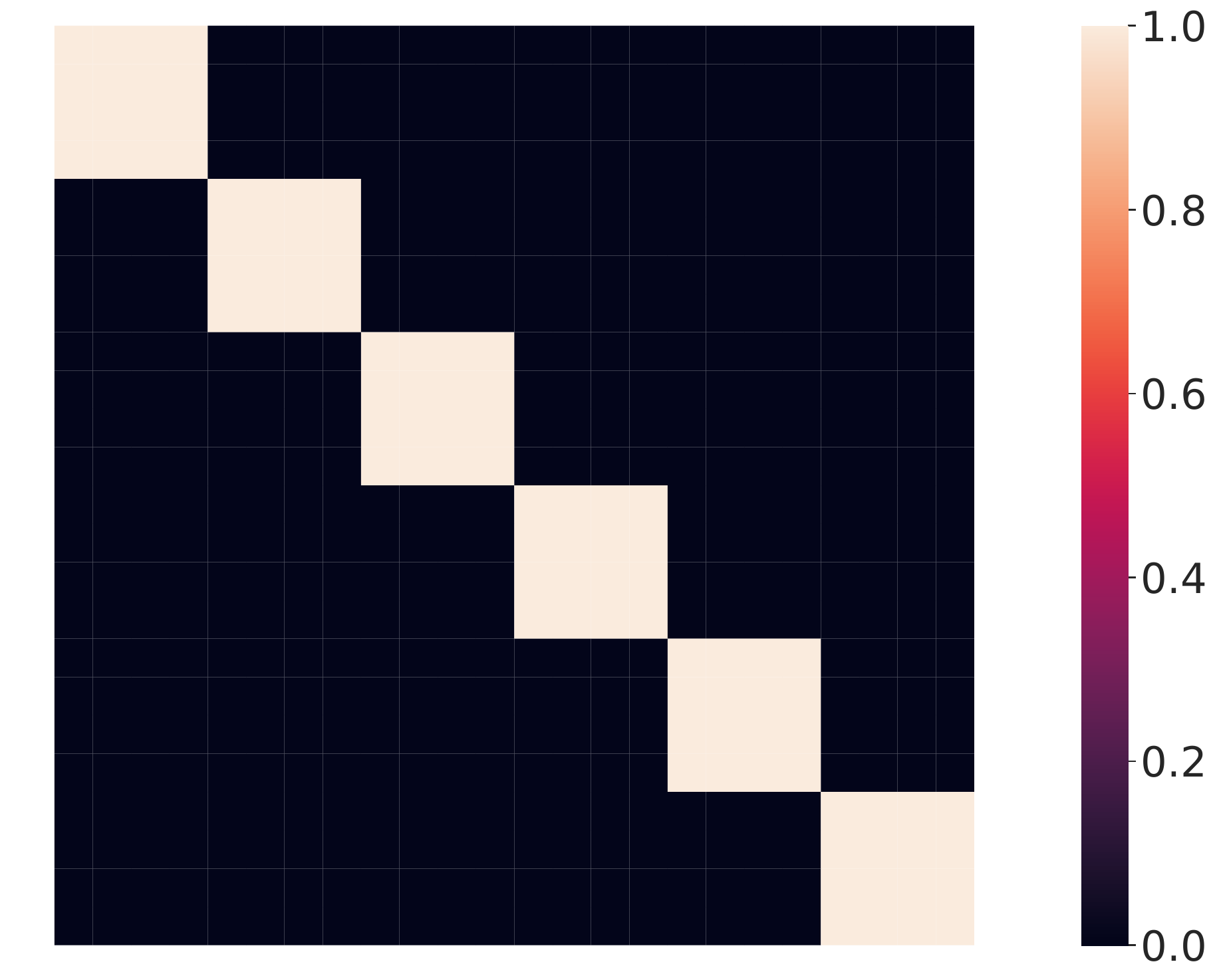}
    \caption{MMCR}
\end{subfigure}
\begin{subfigure}{0.24\linewidth}
    \includegraphics[width=\linewidth]{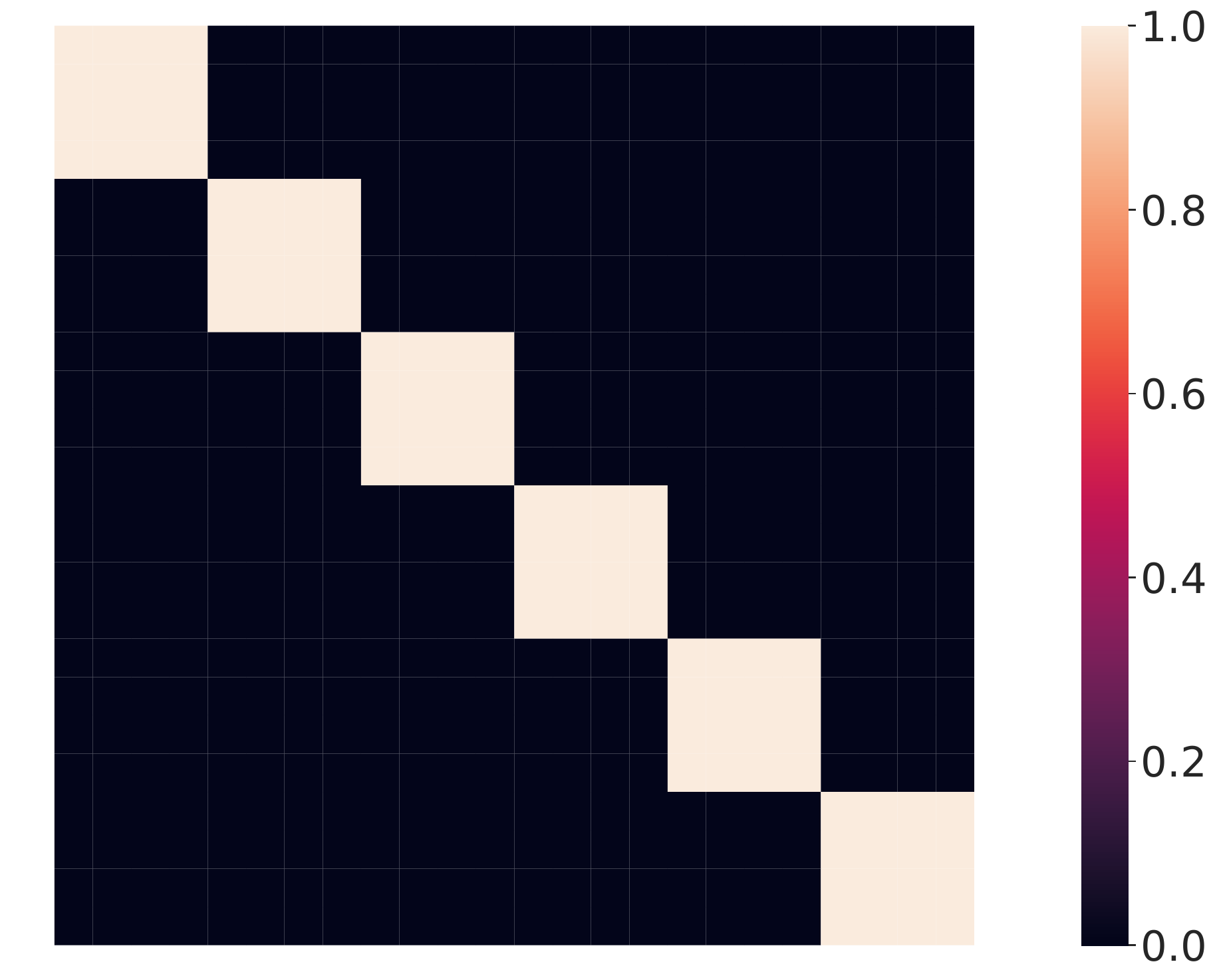}
    \caption{Ours}
\end{subfigure}
\medskip
\begin{subfigure}{0.24\linewidth}
    \includegraphics[width=\linewidth]{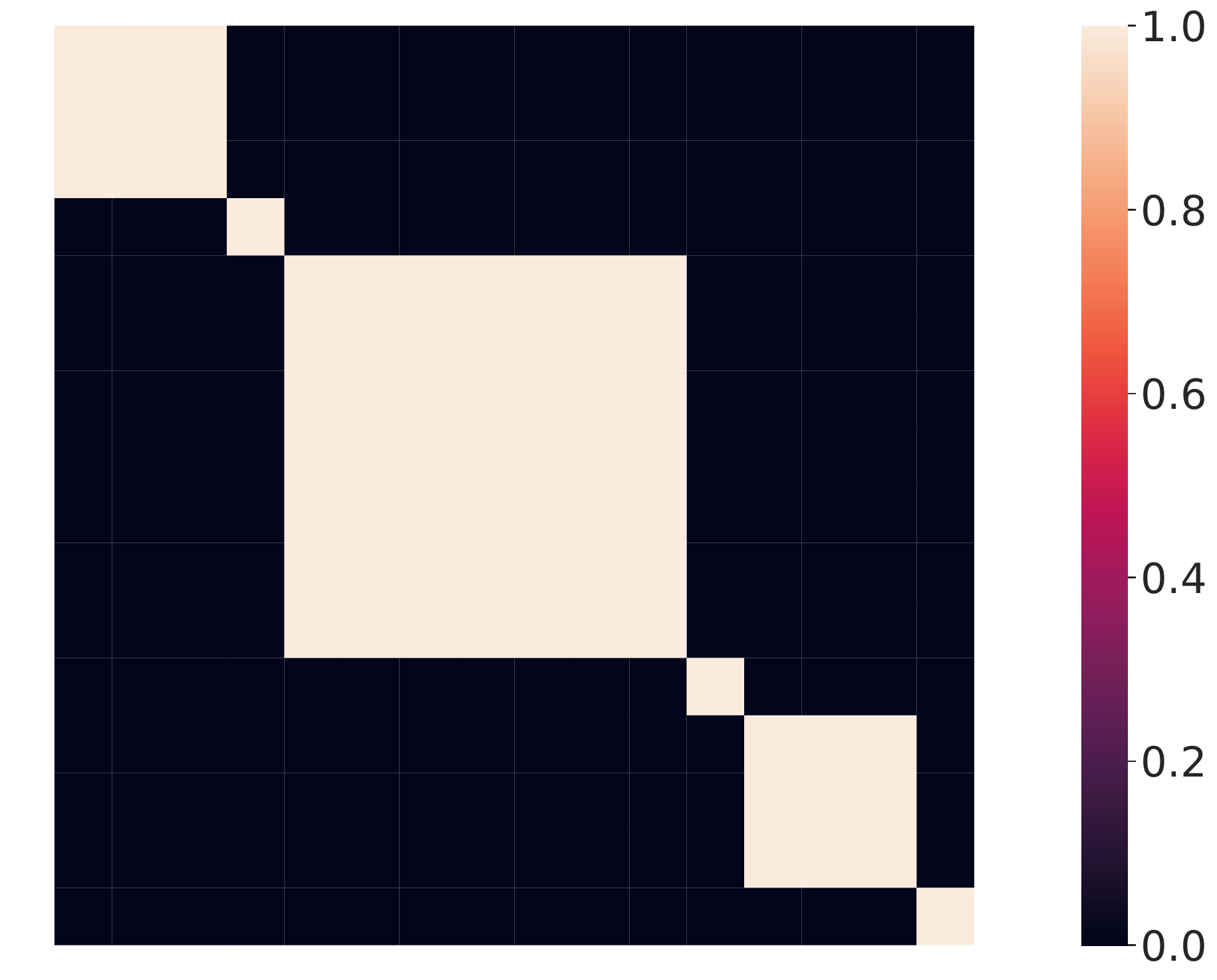}
    \caption{$\Yv^\top \Yv$}
     \label{fig:Y2}
\end{subfigure}
\begin{subfigure}{0.24\linewidth}
    \includegraphics[width=\linewidth]{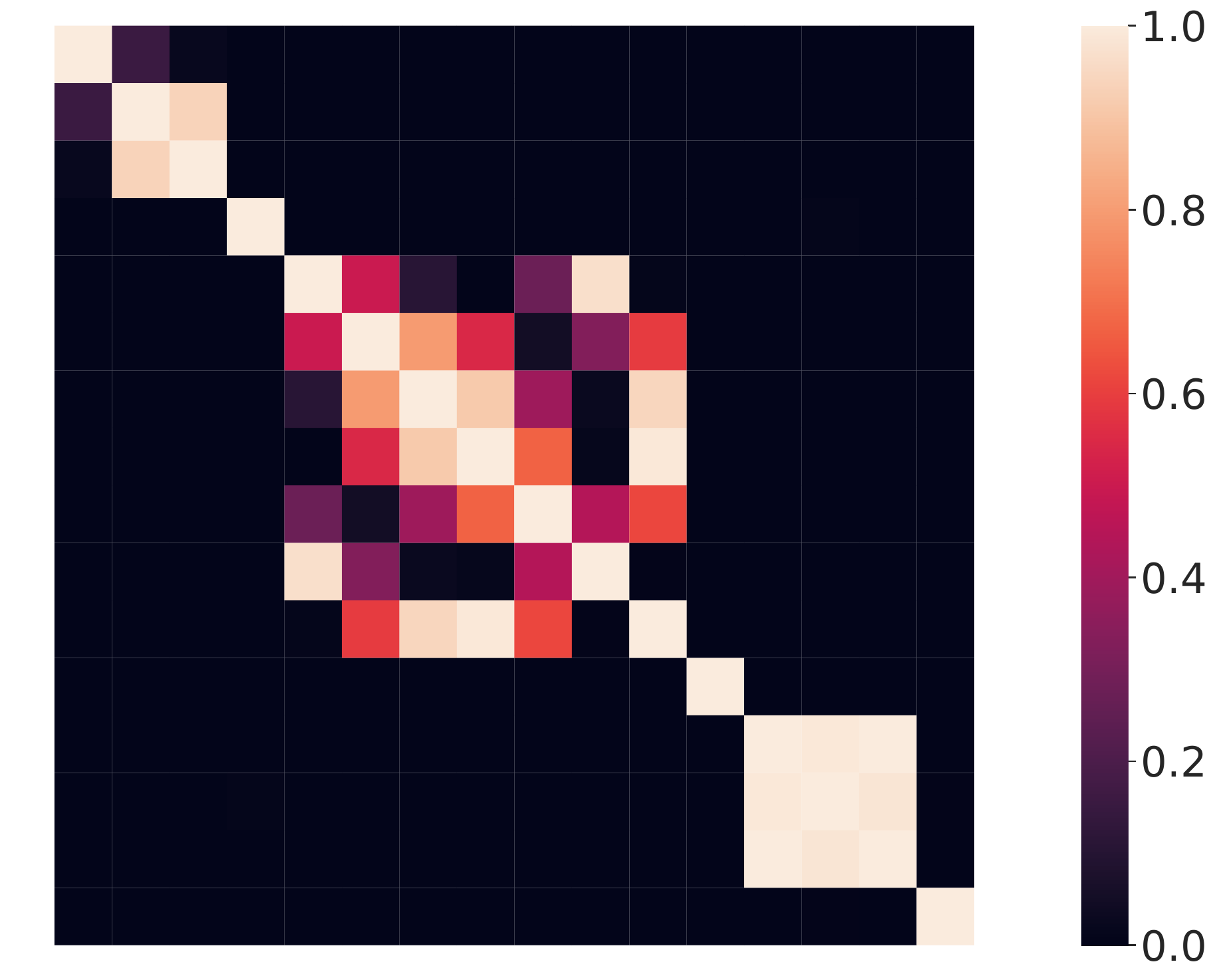}
    \caption{OLE}
\end{subfigure}
\begin{subfigure}{0.24\linewidth}
    \includegraphics[width=\linewidth]{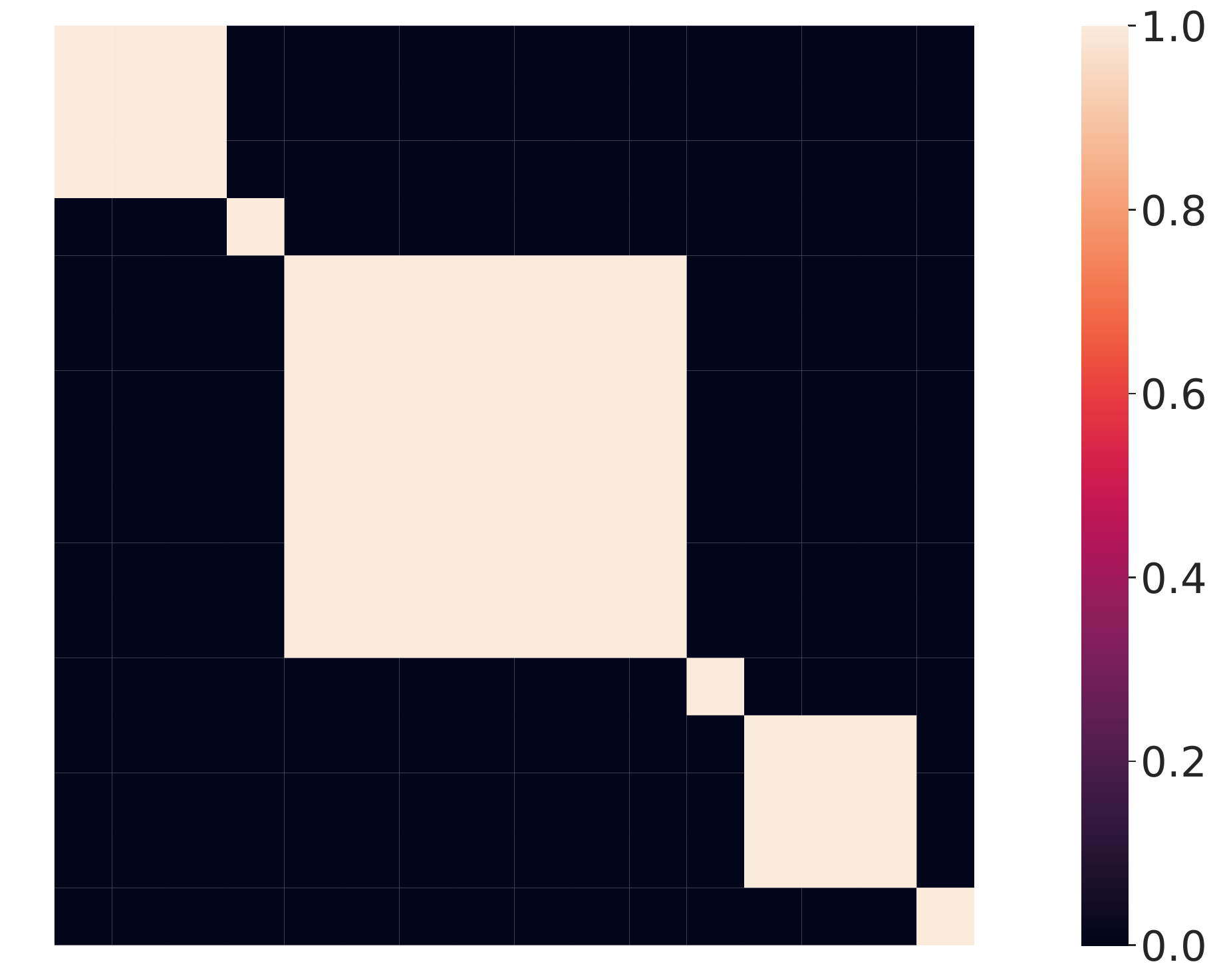}
    \caption{MMCR}
\end{subfigure}
\begin{subfigure}{0.24\linewidth}
    \includegraphics[width=\linewidth]{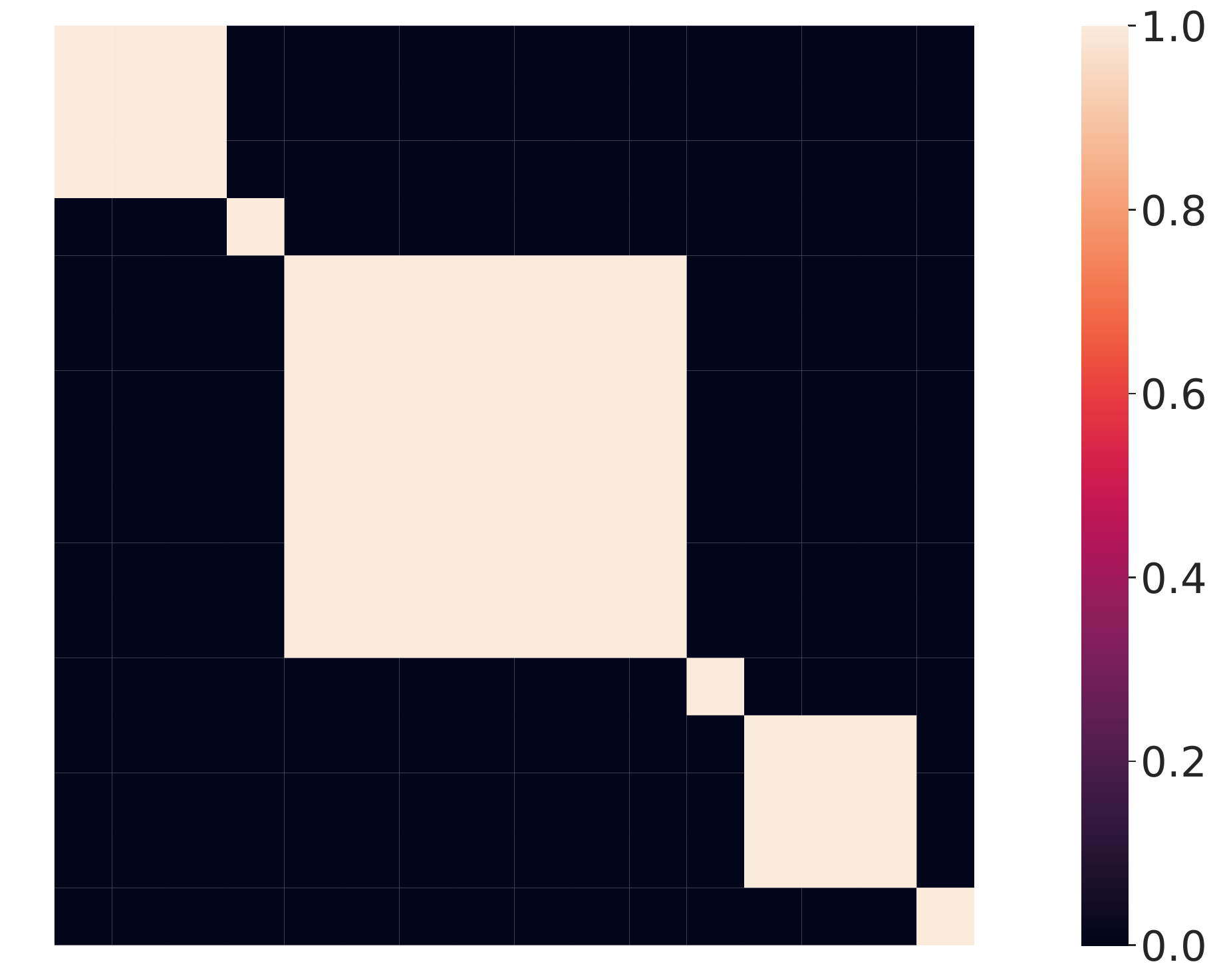}
    \caption{Ours}
\end{subfigure}
\medskip
\begin{subfigure}{0.24\linewidth}
    \includegraphics[width=\linewidth]{fig/Y0_6x15.pdf}
    \caption{$\Yv^\top \Yv$}
     \label{fig:Y0}
\end{subfigure}
\begin{subfigure}{0.24\linewidth}
    \includegraphics[width=\linewidth]{fig/Y0_ole_6x15.pdf}
    \caption{OLE}
\end{subfigure}
\begin{subfigure}{0.24\linewidth}
    \includegraphics[width=\linewidth]{fig/Y0_mcmr_6x15.pdf}
    \caption{MMCR}
\end{subfigure}
\begin{subfigure}{0.24\linewidth}
    \includegraphics[width=\linewidth]{fig/Y0_ours_6x15.pdf}
    \caption{Ours}
\end{subfigure}%
\caption{Gram matrix of $\Yv \in \{0,1\}^{6\times 15}$ and of the representations optimized with OLE, MCMR and our loss, for three synthetic experiments}
\label{fig:all_synthetic_experiments}
\end{figure}

\begin{figure}
\centering
\begin{subfigure}{0.32\linewidth}
    \includegraphics[width=\linewidth]{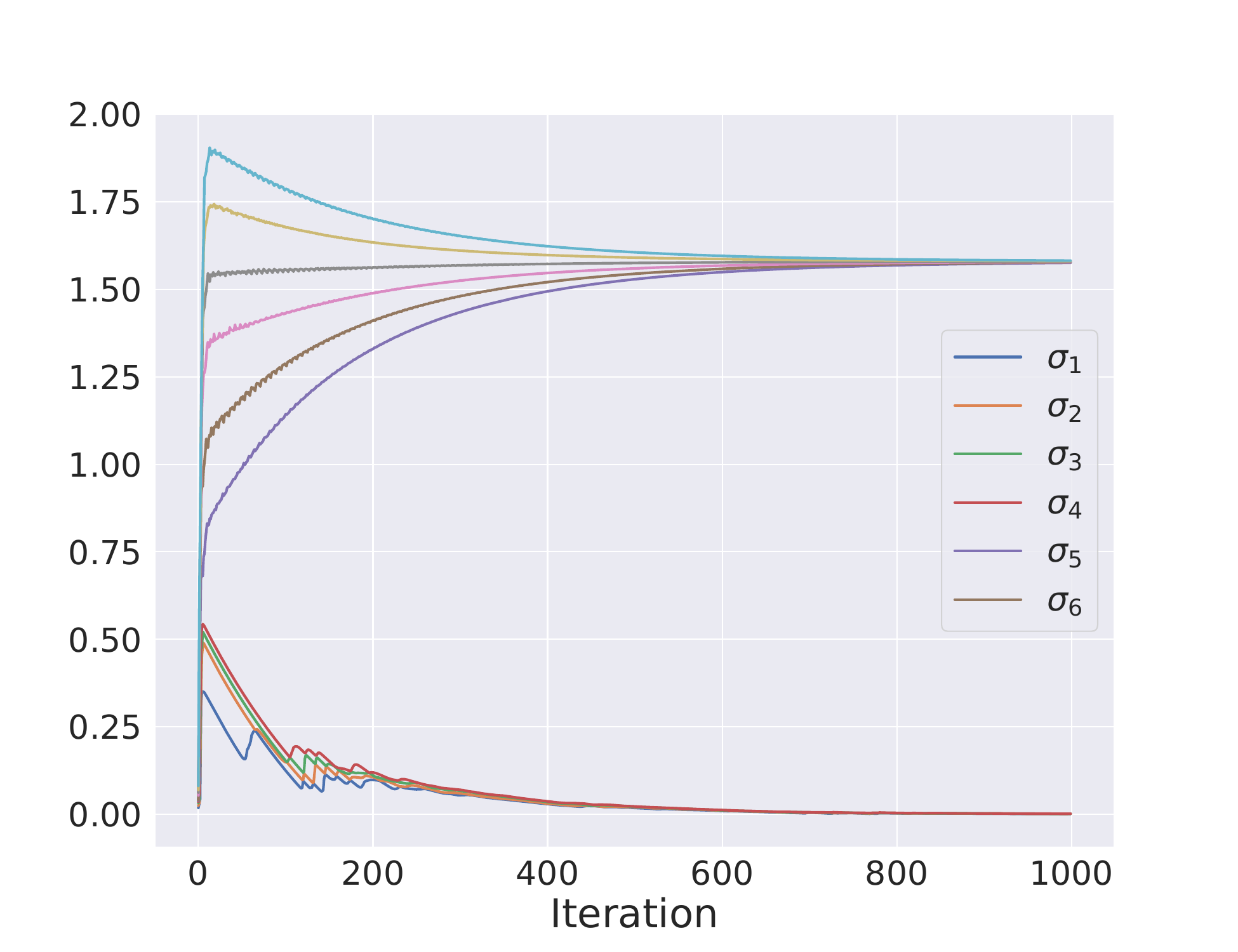}
    \caption{$\Yv$ from Fig. \ref{fig:Y1}}
\end{subfigure}
\begin{subfigure}{0.32\linewidth}
    \includegraphics[width=\linewidth]{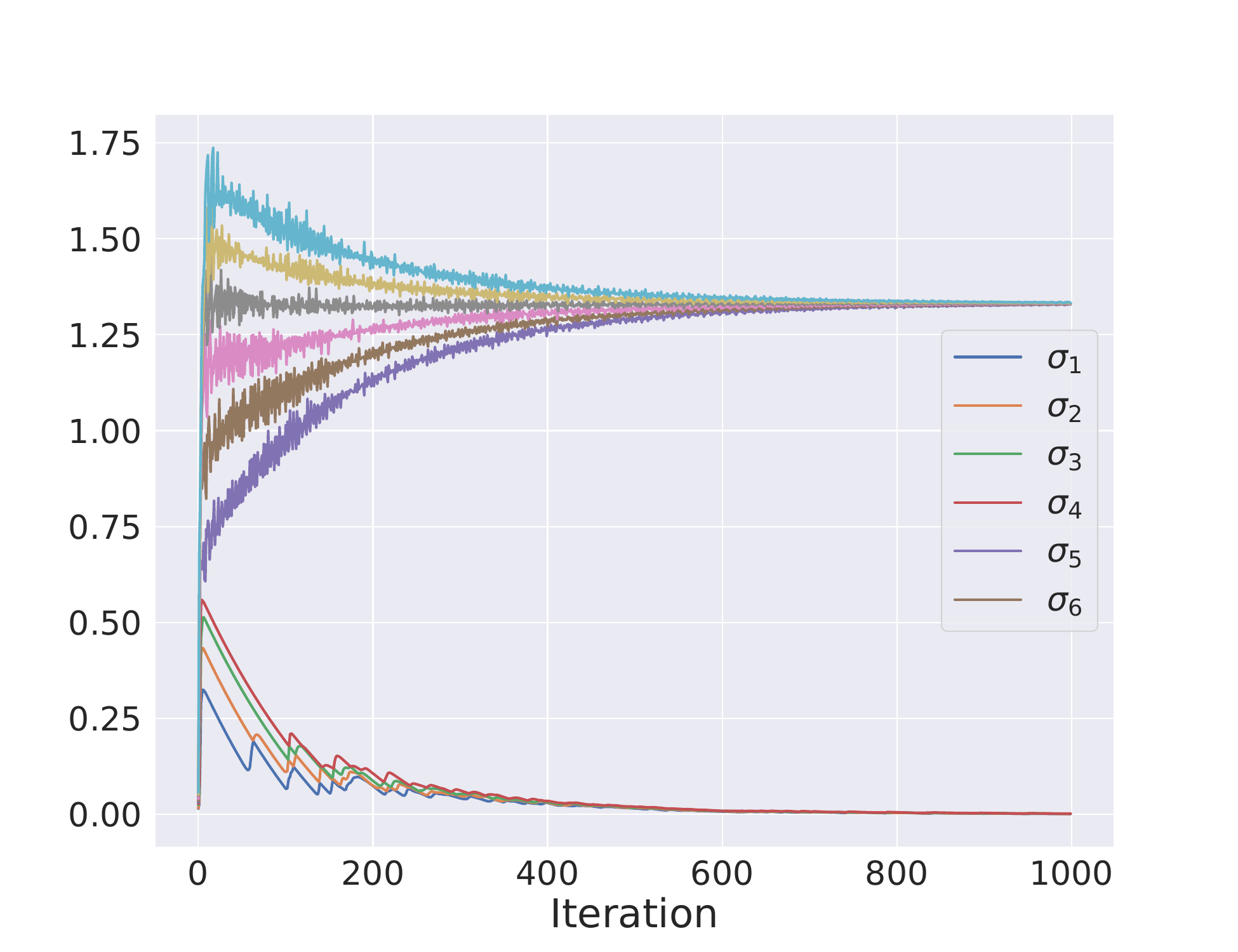}
    \caption{$\Yv$ from Fig. \ref{fig:Y2}}
\end{subfigure}
\begin{subfigure}{0.32\linewidth}
    \includegraphics[width=\linewidth]{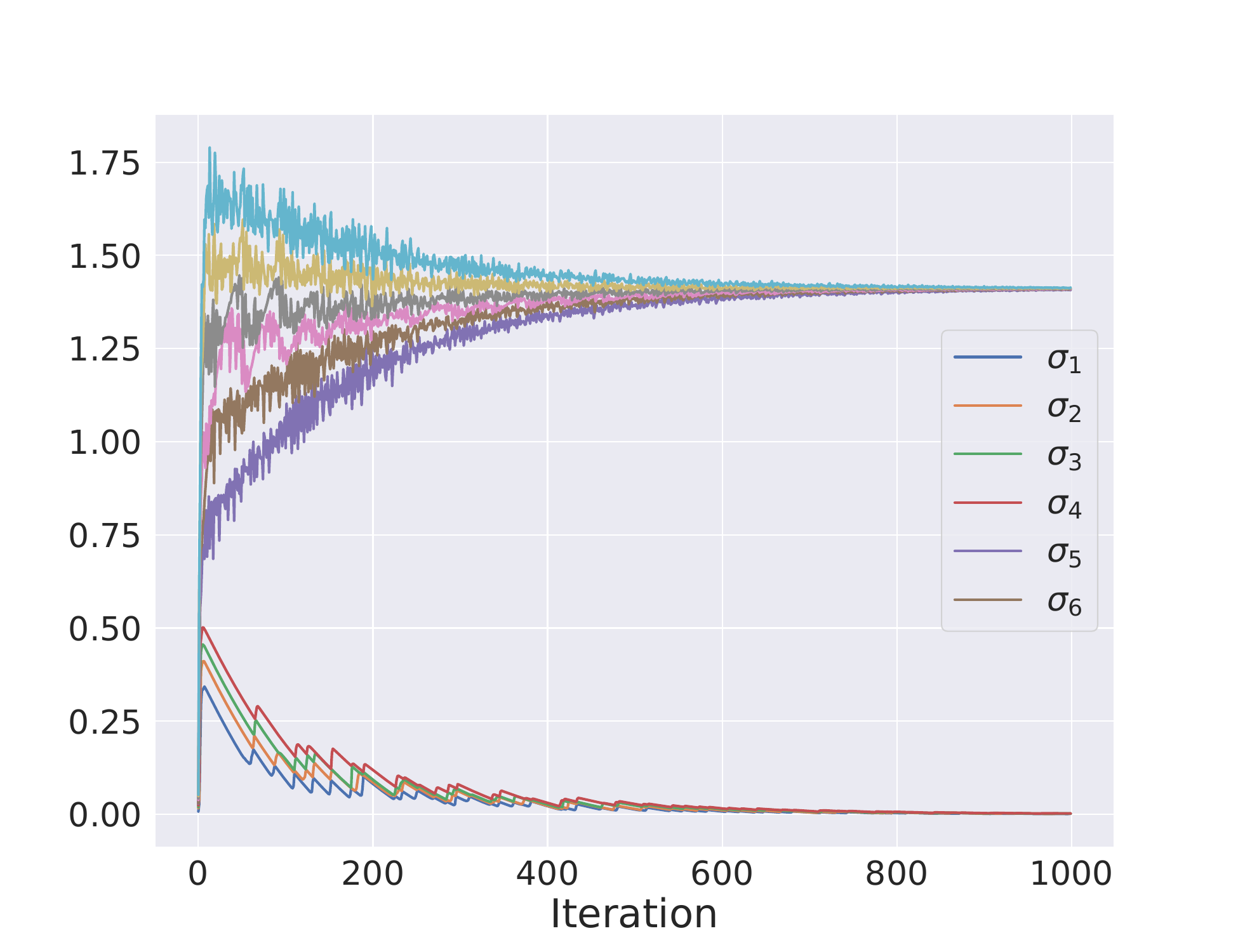}
    \caption{$\Yv$ from Fig. \ref{fig:Y0}}
\end{subfigure}%
\caption{Convergence of the singular values of $\Xv$ during minimization of (\ref{eq:the_loss}).}
\label{fig:synthetic_convergence_svdvals}
\end{figure}

\subsection{Classification Experiments}
\label{sec:app_classification}
\paragraph{MNIST} Dataset of 10,000 grayscale 32x32 images of 10 handwritten digits. 7,000 used for training and 3,000 for testing. No augmentations were used.

\paragraph{FashionMNIST} Dataset of 10,000 grayscale 32x32 images of 10 fashion categories. 7,000 used for training and 3,000 for testing. No augmentations were used.

\paragraph{CIFAR-10} Dataset of 60,000 color 32x32 images corresponding to 10 semantic classes. Training set composed of 50,000 images and test set corresponding to the remaining 10,000. The training set was augmented with random 32x32 crops with 4 pixel padding and random horizontal flips.

\paragraph{CIFAR-100} Dataset of 60,000 color 32x32 images corresponding to 100 semantic classes. Training set composed of 50,000 images and test set corresponding to the remaining 10,000.  The training set was augmented with random 32x32 crops with 4 pixel padding and random horizontal flips.

Training details for the cross-entropy baselines and for our model are provided in Tables \ref{tab:xent_train_details} and \ref{tab:our_train_details}, respectively. 

\begin{table}[]
    \centering
    \caption{Standard classification training with \textbf{crossentropy loss}.}
    \begin{tabular}{ccccccc}
    \toprule
         & \multicolumn{2}{c}{\textbf{MNIST}} & \multicolumn{2}{c}{\textbf{FashionMNIST}} & \textbf{CIFAR10} & \textbf{CIFAR100}\\
        \midrule
        \textbf{Backbone} & ConvNet & ResNet-18 & ConvNet & ResNet-18 & ResNet-18  & ResNet-18 \\
        \textbf{Batch size} & 512 & 512 & 512 & 512 & 512 & 512 \\
        \textbf{Epochs} & 20 & 20 & 50 & 50 & 200 & 200 \\
        \textbf{Optimizer} & Adam & Adam & Adam & Adam & SGD & SGD \\
        \textbf{LR} & 0.01 & 0.01 & 0.01 & 0.01 & 0.1 & 0.1 \\
        \textbf{Scheduler} & Step(1, 0.7) & Step(1, 0.7) & Step(4, 0.5) & Step(4, 0.5) & CosAnneal(200) &  CosAnneal(200)\\
        \midrule
        \textbf{Accuracy} & 0.992 & 0.995 & 0.930 & 0.935 & 0.929 & 0.705 \\
    \bottomrule
    \end{tabular}    \label{tab:xent_train_details}
\end{table}

\begin{table}[]
    \centering
    \caption{Standard classification training with \textbf{our loss}.}
    \begin{tabular}{ccccccc}
    \toprule
         & \multicolumn{2}{c}{\textbf{MNIST}} & \multicolumn{2}{c}{\textbf{FashionMNIST}} & \textbf{CIFAR10} & \textbf{CIFAR100}\\
        \midrule
        \textbf{Backbone} & ConvNet & ResNet-18 & ConvNet & ResNet-18 & ResNet-18  & ResNet-18 \\
        $\mathbf{\alpha}$ & 0.997 & 0.997 & 0.997 & 0.997 & 0.999 & 0.999 \\
        $\mathbf{\beta}$ & 0.05 & 0.05 & 0.05 & 0.05 & 0.01 & 0.01 \\
        \textbf{Batch size} & 512 & 512 & 512 & 512 & 512 & 512 \\
        \textbf{Epochs} & 50 & 50 & 60 & 60 & 200 & 200\\
        \textbf{Optimizer} & SGD & SGD & SGD & SGD & SGD & SGD \\
        \textbf{LR} & 0.0001 & 0.0001 & 0.0001 & 0.0001 & 0.0001 & 0.0001\\
        \textbf{Scheduler} & Step(20, 0.5) & Step(20, 0.5) & Step(30, 0.5) & Step(30, 0.5) & Step(100, 0.1) & Step(100, 0.7)\\
        \midrule
        \textbf{Accuracy} & 0.992 & 0.996 & 0.932 & 0.935 & 0.934 & 0.728 \\
    \bottomrule
    \end{tabular}
    \label{tab:our_train_details}
\end{table}

\subsection{Retrieval Experiments}
\label{sec:app_retrieval}
Table \ref{tab:queries_examples} contains additional examples of propositional queries and the corresponding natural language translations, used to query our model and CLIP, respectively.

\begin{table}[]
    \renewcommand{\arraystretch}{1.0}
    \small
    \centering
    \caption{Examples of propositional and natural language queries used in the Celeb-A retrieval experiment.}
    \begin{tabular}{l}
        $\texttt{Bald}$ \\ \textit{a bald person} \\ \rule{0pt}{3ex}    
        \texttt{Male} \\ \textit{a male person} \\  \rule{0pt}{3ex}    
        \texttt{Eyeglasses} \\ \textit{a person with eyeglasses} \\ \rule{0pt}{3ex}    
        $\neg\texttt{Bald}$ \\ \textit{a person that is not bald} \\ \rule{0pt}{3ex}    
        $\neg\texttt{Wearing Hat}$ \\ \textit{a person not wearing a hat} \\ \rule{0pt}{3ex}    
        $\texttt{Eyeglasses}\land\texttt{Male}$ \\ \textit{a male person with eyeglasses} \\ \rule{0pt}{3ex}    
        $\texttt{Bald}\land\texttt{Eyeglasses}$ \\ \textit{a bald person with eyeglasses} \\ \rule{0pt}{3ex}    
        $\texttt{Eyeglasses} \land\neg \texttt{Wearing Hat}$ \\ \textit{a person with eyeglasses and not wearing a hat} \\ \rule{0pt}{3ex}    
        $\neg\texttt{Wearing Necktie}\land\neg\texttt{Wearing Hat}$ \\ \textit{a person not wearing a necktie and not wearing a hat} \\ \rule{0pt}{3ex}    
        $\texttt{Bald} \land \texttt{Eyeglasses} \land\neg \texttt{Wearing Necktie}$ \\ \textit{a bald person with eyeglasses and not wearing a necktie} \\ \rule{0pt}{3ex}    
        $\neg\texttt{Wearing Necktie}\land\neg\texttt{Wearing Hat}\land\texttt{Male}$ \\ \textit{a male person not wearing a necktie and not wearing a hat} \\ \rule{0pt}{3ex}
        $\texttt{Eyeglasses}\land\texttt{Wearing Necktie}\land\neg\texttt{Wearing Hat}$ \\  \textit{a person with eyeglasses wearing a necktie and not wearing a hat} \\ \rule{0pt}{3ex}  
        $\neg\texttt{Bald}\land\neg\texttt{Eyeglasses}\land\neg\texttt{Wearing Necktie}$ \\ \textit{a person that is not bald, not wearing eyeglasses and not wearing a necktie} \\ \rule{0pt}{3ex}
        $\texttt{Eyeglasses}\land\neg\texttt{Wearing Necktie}\land\neg\texttt{Wearing Hat}\land\texttt{Male}$ \\ \textit{a male person with eyeglasses, not wearing a hat and not wearing a necktie} \\ \rule{0pt}{3ex}    
        $\texttt{Eyeglasses}\land\neg\texttt{Wearing Necktie}\land\neg\texttt{Wearing Hat}\land\neg\texttt{Male}$ \\ \textit{a person that is not male, with eyeglasses, not wearing a hat and not wearing a necktie} \\ \rule{0pt}{3ex}    
        $\neg\texttt{Bald}\land\texttt{Eyeglasses}\land\texttt{Wearing Necktie}\land\neg\texttt{Wearing Hat}\land\texttt{Male}$ \\ \textit{a male person that is not bald, with eyeglasses, wearing a necktie and not wearing a hat} \\ \rule{0pt}{3ex}    
        $\neg\texttt{Bald}\land\neg\texttt{Eyeglasses}\land\neg\texttt{Wearing Necktie}\land\neg\texttt{Wearing Hat}\land\texttt{Male}$ \\ \textit{a male person that is not bald, without eyeglasses, not wearing a hat and not wearing a necktie} \\ \rule{0pt}{3ex}    
        $\neg\texttt{Bald}\land\neg\texttt{Eyeglasses}\land\neg\texttt{Wearing Necktie}\land\neg\texttt{Wearing Hat}\land\neg\texttt{Male}$ \\ \textit{a person that is not male, not bald, without eyeglasses, not wearing a hat and not wearing a necktie}
    \end{tabular}
    \label{tab:queries_examples}
\end{table}